\documentclass[11pt]{article}
\usepackage[letterpaper,margin=1in]{geometry}
\usepackage[T1]{fontenc}
\usepackage{lmodern}
\usepackage{microtype}
\usepackage{titling}
\usepackage{mathtools}
\usepackage{amssymb}
\usepackage{amsthm}
\usepackage{bm}
\usepackage{booktabs}
\usepackage{tabularx}
\usepackage{enumitem}
\usepackage{algorithm}
\usepackage{algpseudocode}
\usepackage[authoryear,round]{natbib}
\usepackage[hidelinks]{hyperref}

\pretitle{\begin{center}\Large\bfseries}
\posttitle{\par\end{center}\vskip 0.5em}
\preauthor{\begin{center}\normalsize\begin{tabular}[t]{c}}
\predate{}
\postdate{}

\hypersetup{
  pdftitle={Adversarial Resilience of Poisson-Process Submodular Maximization over Matroids, and Full-Bandit Learning},
  pdfauthor={Vaneet Aggarwal}
}

\theoremstyle{plain}
\newtheorem{theorem}{Theorem}[section]
\newtheorem{lemma}[theorem]{Lemma}
\newtheorem{proposition}[theorem]{Proposition}
\newtheorem{corollary}[theorem]{Corollary}

\theoremstyle{definition}
\newtheorem{definition}[theorem]{Definition}
\newtheorem{remark}[theorem]{Remark}
\newtheorem{assumption}[theorem]{Assumption}

\newcommand{\OPT}{\mathrm{OPT}}
\newcommand{\Mar}{\mathrm{Mar}}
\newcommand{\I}{\mathcal I}
\newcommand{\E}{\mathbb E}
\newcommand{\Prb}{\mathbb P}

\newcommand{\eps}{\varepsilon}
\newcommand{\xierr}{\xi}
\newcommand{\one}{\mathbf 1}

\newcommand{\argmax}{\operatorname*{arg\,max}}

\newcommand{\Ipl}{\I^{+1}}

\title{Adversarial Resilience of Poisson-Process Submodular Maximization 
over Matroids, and Full-Bandit Learning}

\author{Vaneet Aggarwal\\
  Purdue University
}
\date{}

\begin{document}

\maketitle

\begin{abstract}
We study nonnegative submodular maximization on $n$ elements subject to a
general matroid of rank $k$, when the
offline algorithm is given an arbitrary controlled value oracle.  Our main result is
an adversarial resilience theorem for the Spiteful Greedy Swap Poisson Process
(SGS-Poisson): without modifying its Poisson intensity, single-element exchange
rule, or spiteful drop step, the algorithm retains limiting approximation factors
$1/e$ for non-monotone objectives and $1-1/e$ for monotone objectives.  More
precisely, given an error bound $\xierr\ge0$, under every controlled oracle
$\widehat f$ satisfying
$|\widehat f(S)-f(S)|\le \xierr$ for every set $S$, our implementation returns a
feasible set with expected value at least
$(1/e-\varepsilon)\OPT-O(k\xierr)$ and
$(1-1/e-\varepsilon)\OPT-O(k\xierr)$, respectively, where $\OPT$ is the
feasible optimum. The implementation uses a \emph{deterministically bounded} budget of
$O(nk^{2}\varepsilon^{-2}\log n\log^{2}(1/\varepsilon))$ oracle calls.
As a consequence, an offline-to-online reduction yields full-bandit
combinatorial multi-armed bandit (CMAB) algorithms for
general matroid-constrained submodular rewards with exact limiting
approximation-regret factors $1/e$ and $1-1/e$ and
$\widetilde O(n^{1/5}k^{4/5}T^{4/5})$ regret over $T$ rounds.
These online guarantees allow exploration to play sets that become independent
after deleting at most one element; exploitation and the benchmark remain
matroid-feasible. For unit-capacity partition matroids we obtain
$\widetilde O(n^{1/5}k^{3/5}T^{4/5})$ under the same exploration relaxation.
We establish a deterministic query budget by truncating the Poisson process
and identify the enlarged action set needed to answer its offline queries.
\end{abstract}

\begingroup
\par\endgroup

\section{Introduction}

Submodular objectives capture diminishing returns in applications such as
influence maximization, sensor placement, experimental design, and data
summarization \citep{KrauseGolovin2014}. Matroid constraints express which
collections of items can be selected, extending cardinality budgets to
structured independence requirements. For a ground set $U$ of $n$ items and
a matroid $M=(U,\I)$ of rank $k$, we study
\[
    \OPT:=\max_{S\in\I}f(S),\qquad f:2^U\to[0,1],
\]
where $f$ is nonnegative and submodular. With exact value access, continuous
greedy attains $1-1/e$ for monotone objectives \citep{CCPV2011}, and measured
continuous greedy attains $1/e$ for non-monotone objectives
\citep{Feldman2011}. Combinatorial approaches also attain the classical
monotone factor \citep{FilmusWard2014,BuchbinderFeldman2024}.

In learning applications, however, the values supplied to an optimization
routine are empirical estimates rather than exact evaluations. Under
full-bandit feedback, a learner observes only the total reward of its chosen
set, not item-level contributions. Even when the mean reward is submodular,
individual observations and the resulting empirical oracle need not be
submodular or monotone. An exact-oracle approximation theorem therefore
does not automatically yield an online guarantee. The offline-to-online
framework of \citet{Nie2023,FouratiFramework2024} addresses this gap through
\emph{resilience}: the offline approximation must survive uniformly bounded
errors in its value queries.

The Poisson-process approach of \citet{Ganz2026,Kulik2026} provides an
attractive candidate for such a resilient primitive. In particular, the
Spiteful Greedy Swap Poisson Process (SGS-Poisson) of \citet{Kulik2026}
maintains a feasible set through single-element exchanges and a randomized
drop step, attaining both $1/e$ and $1-1/e$ in the exact-oracle setting.
We ask whether this same process remains effective when all of its decisions
are driven by a persistent, adversarially controlled oracle. The difficulty
is not just error in one marginal comparison: a perturbation can change the
selected residual base, its exchange map, an adaptive stopping time, and
hence the entire subsequent trajectory.

\paragraph{Resilient offline optimization.}
For every target loss $\eps\in(0,1/2]$, supplied radius $\xierr\ge0$, and
persistent oracle $\widehat f$ satisfying
\begin{equation}
|\widehat f(S)-f(S)|\le\xierr\qquad\forall S\subseteq U,
\label{eq:oracle-control}
\end{equation}
we construct an implementation returning $A_{\rm out}\in\I$ with
\begin{equation}
\E[f(A_{\rm out})]\ge
\begin{cases}
(1/e-\eps)\OPT-Ck\xierr,&f\text{ non-monotone},\\[1mm]
(1-1/e-\eps)\OPT-Ck\xierr,&f\text{ monotone},
\end{cases}
\label{eq:main-guarantee}
\end{equation}
where $C=8$ suffices (Theorem~\ref{thm:main}). Its controlled-oracle budget is
bounded \emph{deterministically} by
\begin{equation}
O\!\left(
 nk\log\frac1\eps+
 \frac{nk^2\log n\log(1/\eps)}{\eps^2}
 +\frac{nk\log^2(1/\eps)}{\eps^2}
\right).
\label{eq:main-complexity}
\end{equation}
The Poisson intensity, single-element exchange mechanism, and spiteful drop
are inherited from SGS-Poisson, not introduced here. Our contribution is to
establish their controlled-oracle resilience and the implementation properties
needed for learning; the process is truncated only to enforce its query budget.

\paragraph{Preserving a certificate along the noisy trajectory.}
The central argument evaluates the \emph{true} objective along the
controlled-oracle execution, rather than trying to couple it to an
exact-oracle execution. We first combine the standard Residual Random Greedy
(RRG) exchange coupling \citep{BFS2014} with a robust maximum-base comparison.
Amplification gives a realized optimum estimate $\widehat V$ that can drive
the stopping rule of Advanced Preprocessing
(Lemmas~\ref{lem:rrg} and~\ref{lem:certificate}). Dummy augmentation makes
explicit why residual maximum-base comparisons remain legitimate even for
a non-submodular surrogate (Lemma~\ref{lem:dummy}).

The key step is then to preserve the exchange potential
\begin{equation}
\Phi_t=f(Q_t\cup O_t)+\tfrac12f(Q_t)
\label{eq:intro-potential}
\end{equation}
under both noisy base selection and noisy adaptive stopping. On the
high-probability optimum-certificate event, before preprocessing stops,
\begin{equation}
\E[\Phi_{t+1}-\Phi_t\mid\mathcal F_t]
\ge\frac{8\OPT-k\xierr}{k-t}.
\label{eq:intro-drift}
\end{equation}
This residual-rank drift inequality (Lemma~\ref{lem:drift}) yields the
contracted-instance certificate required by the Poisson analysis
(Theorem~\ref{thm:preprocess}). It absorbs oracle error without charging it
anew for every decision along a potentially different adaptive trajectory.

\paragraph{Robust swaps and a usable learning interface.}
The second ingredient concerns a \emph{sum} of sampled multilinear marginals,
not a single marginal. Following the base-sum concentration strategy of
\citet{Kulik2026}, we separate sampling error in the true objective from
bias due to the controlled oracle. The resulting swap remains valid and
right-continuous, with almost-above-average error
\begin{equation}
\eta\le c_1\eps\OPT+c_2k\xierr
\label{eq:intro-eta}
\end{equation}
(Lemma~\ref{lem:swap}). The relative term $\eps\OPT$, rather than an absolute
$O(\eps)$ loss, is essential for resilience. The Poisson differential
inequality itself is an imported result of \citet{Kulik2026}.

Two additional implementation issues must be addressed before invoking the
learning reduction. First, the number of Poisson events is unbounded;
Lemma~\ref{lem:budget} caps it at a deterministic budget with expected loss
at most $\rho\OPT$. Second, feasibility of the maintained set does not imply
feasibility of every query: the swap estimator can evaluate a feasible set
with one additional element. We prove that all queries lie in $\Ipl$, the
family of sets that become independent after deleting at most one element
(Lemma~\ref{lem:query-feasibility}). Our online guarantees therefore allow
exploration in $\Ipl$, while exploitation and the benchmark remain in $\I$.
This relaxation is explicit; we do not claim strictly feasible exploration.
The reduction is further checked for cached, adaptive empirical queries and
a supplied error radius in Section~\ref{sec:cmab}.

\paragraph{Full-bandit learning and partition matroids.}
These ingredients yield $\widetilde O(n^{1/5}k^{4/5}T^{4/5})$ regret at the
exact limiting benchmark factors $1/e$ and $1-1/e$ for non-monotone and
monotone mean rewards, respectively (Corollary~\ref{cor:cmab}). Within the
offline-to-online framework, the earlier general-matroid monotone result of
\citet{Nie2025KSubmodular} attains $1/2$; their non-monotone result requires
at least two labels and does not specialize to ordinary set-submodularity.
Our improvement in approximation factor is not a claim to dominate
specialized cardinality-bandit rates. For unit-capacity partition matroids,
we use the lazy per-part swap of \citet{Kulik2026} and add a controlled-oracle
analysis and a deterministic query-cost cap. This reduces the rank dependence
to $\widetilde O(n^{1/5}k^{3/5}T^{4/5})$ under the same exploration relaxation
(Theorem~\ref{thm:partition}).

Our offline noise model is persistent adversarial error bounded uniformly
by $\xierr$, rather than the distribution-specific persistent stochastic
noise models studied by \citet{HassidimSinger2017,Bhawalkar2025}. The distinction
lets the offline theorem apply to empirical reward tables without imposing
submodularity on their realized values. The proofs below retain the
controlled-oracle and implementation arguments, while using established
exchange, sampling, and Poisson results with explicit attribution.

\section{Related Work}
\label{sec:related}

We review the offline and online results most closely related to our contribution.

\subsection{Submodular Maximization under Matroid Constraints}

For monotone submodular maximization under a matroid, continuous greedy gives the
classical $1-1/e$ approximation \citep{CCPV2011}. For non-monotone objectives,
measured continuous greedy gives $1/e$ \citep{Feldman2011}, while the recent
SGS-Poisson line obtains the same limiting non-monotone factor and the
$1-1/e$ monotone factor through a discrete Poisson-process construction
\citep{Ganz2026,Kulik2026}. Our base algorithm is exactly SGS-Poisson: its
Poisson intensity, valid-swap rule, and spiteful drop are not modified.
Residual Random Greedy (RRG) was introduced by \citet{BFS2014};
we use its matroid version as a constant-factor preprocessing primitive.

\subsection{Offline-to-Online Learning and Full-Bandit CMAB}
\citet{Nie2023} introduced a black-box offline-to-online framework
for stochastic combinatorial multi-armed bandits with full-bandit feedback.
The key condition is robustness of the offline approximation algorithm to
bounded errors in value-oracle evaluations.  Under this condition, the offline
algorithm can be used as a black box and its approximation factor becomes the
benchmark for sublinear $\alpha$-regret.  The framework was instantiated for
several submodular settings, including cardinality and knapsack constraints.
Special cases of this approach appear in
\citet{Nie2022,Fourati2023,Nie2025KSubmodular}.
The $k$-submodular results in \citet{Nie2025KSubmodular} include ordinary
set-submodularity as the $k_{\rm sub}=1$ special case, but only in the monotone
regime. Thus for \emph{monotone} objectives, general-matroid full-bandit CMAB is
not new in itself, and the distinction here is the approximation factor
($1-1/e$ versus $1/2$) at the price of a worse horizon exponent
($T^{4/5}$ versus $T^{2/3}$). The approximation factors and horizon exponents must be distinguished:
a sublinear $1/2$-regret bound does not imply sublinear $(1-1/e)$-regret.
Conversely, a sublinear $(1-1/e)$-regret bound does imply sublinear
$1/2$-regret, but not necessarily its faster $T^{2/3}$ rate.

For \emph{non-monotone} objectives, the matroid guarantee of
\citet{Nie2025KSubmodular} requires $k_{\rm sub}\ge2$ and does not specialize
to ordinary set-submodularity. Their proof, following \citet{Sun2022}, uses
pairwise monotonicity with a second label $h\neq i$ and full-rank optimality.
Neither argument supplies the claimed guarantee when $k_{\rm sub}=1$.
For example, on the rank-$2$ uniform matroid over $\{a,b\}$, the nonnegative
submodular function with $f(\{a\})=1$ and
$f(\varnothing)=f(\{b\})=f(\{a,b\})=0$ has $\OPT=1$, whereas a greedy
algorithm that fills a base returns value $0$. We therefore retain only their
monotone specialization in Table~\ref{tab:cmab-related}. Our non-monotone
CMAB result uses the enlarged exploration action set in
Assumption~\ref{ass:action-set}; it does not resolve the strictly feasible
full-bandit problem.

The query-feasibility obstruction itself predates the Poisson approach.
\citet{Zhang2019} study a responsive bandit model allowing infeasible value
queries but assigning zero reward to those rounds. Our exploration relaxation
instead restricts queries to $\Ipl$ and observes rewards with mean $f(S)$
there. The contribution here is the specific query-domain certificate and
its use in the stochastic resilience reduction, not the general observation
that offline queries may be infeasible online.

\citet{FouratiFramework2024} subsequently extended
this interface to offline algorithms whose guarantee is $(\alpha-\varepsilon)$
and whose oracle complexity depends polynomially on $1/\varepsilon$.  They remove
the offline $\varepsilon$ loss and yield $\alpha$-regret.
We use the resulting explore-then-commit framework for a single learner.
A special case of this framework can be seen in \citet{Fourati2024}.
We use the approach of \citet{FouratiFramework2024} in our CMAB result.

Table~\ref{tab:cmab-related} summarizes these full-bandit results.  We list
only CMAB results and only constraints that are matroids or special cases of
matroids.  The cardinality constraint is the uniform matroid, and the
unconstrained problem is the free matroid. Our two rows give the monotone
and non-monotone results separately.

\begin{table}[t]
\centering
\caption{Stochastic full-bandit CMAB results for ordinary set-submodular rewards under matroid constraints or special cases.}
\label{tab:cmab-related}
\small
\begin{tabularx}{\textwidth}{@{}>{\raggedright\arraybackslash}p{2.2cm}>{\raggedright\arraybackslash}p{2.4cm}>{\raggedright\arraybackslash}p{1.8cm}>{\raggedright\arraybackslash}p{1.3cm}>{\raggedright\arraybackslash}X@{}}
\toprule
Paper & Constraint & Objective & Approx. & Regret \\
\midrule
\citet{Nie2022}
& Uniform (cardinality)
& Monotone
& $1-1/e$
& $\mathcal O(n^{1/3}k^{4/3}T^{2/3}\log^{1/2}T)$ \\

\citet{Fourati2024}
& Uniform (cardinality)
& Monotone
& $1-1/e$
& $\widetilde{\mathcal O}(n^{1/3}k^{2/3}T^{2/3})$ \\

\citet{Fourati2023}
& Free
& Non-monotone
& $1/2$
& $\widetilde{\mathcal O}(nT^{2/3})$ \\

\citet{FouratiFramework2024}
& Uniform (cardinality)
& Non-monotone
& $1/e$
& $\widetilde{\mathcal O}(n^{1/5}k^{2/5}T^{4/5})$ \\

\citet{Nie2025KSubmodular}
& General matroid
& Monotone
& $1/2$
& $\widetilde{\mathcal O}(n^{1/3}kT^{2/3})$ \\

\textbf{This work}
& \textbf{General matroid}
& \textbf{Monotone}
& $\mathbf{1-1/e}$
& $\mathbf{\widetilde{\mathcal O}(n^{1/5}k^{4/5}T^{4/5})}$ \\

\textbf{This work}
& \textbf{General matroid}
& \textbf{Non-monotone}
& $\mathbf{1/e}$
& $\mathbf{\widetilde{\mathcal O}(n^{1/5}k^{4/5}T^{4/5})}$ \\
\bottomrule
\end{tabularx}
\begin{minipage}{0.96\linewidth}
\footnotesize
\emph{Note.} The monotone result of \citet{Nie2025KSubmodular} is specialized
from $k_{\rm sub}$-submodular bandits to $k_{\rm sub}=1$. Here $k$ denotes
matroid rank, not the number of labels $k_{\rm sub}$. Their non-monotone
matroid guarantee requires $k_{\rm sub}\ge2$ and is therefore omitted.
Our rows assume that exploration may play sets in $\Ipl$
(Assumption~\ref{ass:action-set}); exploitation and the benchmark remain in
$\I$. Thus they are not guarantees for strictly feasible exploration.
For unit-capacity partition matroids, our rank dependence improves to
$k^{3/5}$ (Theorem~\ref{thm:partition}).
\end{minipage}
\end{table}

\section{Problem Setup and Resilience}
\label{sec:setup}

All logarithms are natural unless their base is indicated. In complexity
bounds, $\log n$ means $\log\max\{2,n\}$.
Let $M=(U,\I)$ be a matroid of rank $k$ and $n=|U|$. A set function
$f:2^U\to[0,1]$ is submodular if
\[
f(A)+f(B)\ge f(A\cup B)+f(A\cap B)
\qquad\forall A,B\subseteq U.
\]
We use the equivalent diminishing-returns notation
\[
f(i\mid S):=f(S\cup\{i\})-f(S),
\]
for which $S\subseteq T$ and $i\notin T$ imply
$f(i\mid S)\ge f(i\mid T)$. Define
\[
    \OPT:=\max_{S\in\I}f(S).
\]
We assume standard access to an independence oracle for the matroid. Maximum-weight
bases in contracted matroids are computed using this access. As usual, we count
only value-oracle calls in the resilience and CMAB complexity bounds; independence
queries and the computation of maximum-weight bases are not counted.

For the online full-bandit application, at round $t$ the learner chooses an
admissible super-arm $A_t$ and observes only an aggregate reward $Y_t\in[0,1]$.
Each play of $S$ supplies an independent draw from a fixed distribution of mean
$f(S)$, independent of past observations and the learner's randomization.
Only the mean function is assumed submodular; no pathwise submodularity
condition is imposed on the reward observations. In particular,
$\E[Y_t\mid\mathcal H_{t-1},A_t]=f(A_t)$, where $\mathcal H_{t-1}$ is the
history before round $t$. For $\alpha\in(0,1]$, we define the
cumulative $\alpha$-regret by
\begin{equation}
R_\alpha(T):=\alpha T\OPT-\mathbb E\left[\sum_{t=1}^T Y_t\right].
\label{eq:regret}
\end{equation}
Only the aggregate reward is observed; no component-wise or semi-bandit
feedback is assumed. The precise action set available to the learner is
discussed in Section~\ref{sec:feasible-queries}.

\subsection{Controlled Oracle}

A $\xierr$-controlled oracle is a function
$\widehat f:2^U\to\mathbb R$ satisfying
\begin{equation}
|\widehat f(S)-f(S)|\le\xierr
    \qquad\forall S\subseteq U.
    \label{eq:controlled-oracle}
\end{equation}
The perturbation is deterministic and may be adversarial; no stochastic or
independence assumption is imposed. Because $\widehat f$ is a fixed function,
the same set always receives the same oracle value; thus the adversarial
perturbation is persistent throughout the execution.

We may project $\widehat f$ onto $[0,1]$ without increasing its error. Hence,
throughout the proof we assume $0\le\widehat f(S)\le1$. Then
\begin{equation}
|\widehat f(i\mid S)-f(i\mid S)|\le2\xierr.
\label{eq:oracle-marginal}
\end{equation}
The following sharper form is what we actually use when comparing two
equal-cardinality sets of marginals taken at a \emph{common} base point; it
saves a factor of two because the $\widehat f(S)$ terms cancel.

\begin{lemma}[Common-base comparison]
\label{lem:common-base}
Let $S\subseteq U$ and let $X,Y\subseteq U\setminus S$ satisfy $|X|=|Y|=r$. Then
\begin{equation}
\left|
\sum_{u\in X}\widehat f(u\mid S)-\sum_{u\in Y}\widehat f(u\mid S)
\;-\;
\Bigl(\sum_{u\in X}f(u\mid S)-\sum_{u\in Y}f(u\mid S)\Bigr)
\right|
\le 2r\xierr.
\label{eq:common-base}
\end{equation}
\end{lemma}

\begin{proof}
Both differences equal
$\sum_{u\in X}h(S+u)-\sum_{u\in Y}h(S+u)$ for $h\in\{\widehat f,f\}$, because
the $r\,h(S)$ terms cancel. Each of the $2r$ evaluations of $h$ differs by at
most $\xierr$ between $h=\widehat f$ and $h=f$.
\end{proof}

\subsection{Resilience Definition}

We use the $(\alpha,\beta,\gamma,\psi,\delta)$ resilience parameterization of
\citet{FouratiFramework2024}, which extends the robust-approximation notion of
\citet{Nie2023} by recording the accuracy-dependent oracle complexity.

\begin{definition}[$(\alpha,\beta,\gamma,\psi,\delta)$-resilience]
\label{def:resilience}
For $\eps\in(0,1/2]$, an offline algorithm $\mathcal A(\eps;\xierr)$ is
called resilient with parameters $(\alpha,\beta,\gamma,\psi,\delta)$ if,
for every advertised error radius $\xierr\ge0$, under a
$\xierr$-controlled oracle it outputs a feasible
$\Theta$ satisfying
\begin{equation}
\E[f(\Theta)]
    \ge
    (\alpha-\eps)\OPT-\delta\xierr,
    \label{eq:resilience-guarantee}
\end{equation}
and makes at most
\begin{equation}
N(\eps)=\psi\,\eps^{-\beta}\log^\gamma(1/\eps)
\label{eq:resilience-complexity}
\end{equation}
oracle calls, up to universal constant factors, uniformly in $\xierr$.
We allow the algorithm to know this radius, as needed by
\eqref{eq:certificate-estimator}. In the online reduction the radius is
computed from the sampling budget, not from knowledge of $f$.
\end{definition}

\begin{remark}[On the query bound]
\label{rem:deterministic-N}
We read \eqref{eq:resilience-complexity} as a \emph{deterministic} bound on the
number of oracle calls, because this is how it is used in the reduction: the
online algorithm plays each queried action $r^\star$ times and charges
$N(\eps)r^\star$ rounds to exploration. A bound that holds only in expectation
is not sufficient for that accounting. The SGS-Poisson process makes a random
number of swap calls, so we truncate it in Section~\ref{sec:budget} to obtain a
deterministic budget.
\end{remark}

\begin{remark}[On $\gamma$, $\eps$, and the error radius]
\label{rem:gamma}
\citet{FouratiFramework2024} state their framework for $\gamma\in\{0,1\}$;
our query bound has $\gamma=2$. The same regret exponents extend to any fixed
$\gamma\ge0$ when the exploration parameters include the corresponding
logarithmic factors. Theorem~\ref{thm:fourati} supplies this tuning explicitly,
along with the advertised error radius used by our offline routine.
We require no invocation at $\eps=0$. If the tuned accuracy exceeds $1/2$,
the desired bound is already trivial up to its logarithmic factors, and the
learner can play $\varnothing$ instead of invoking the offline algorithm.
\end{remark}

\begin{remark}[Unknown horizon]
\label{rem:anytime}
For unknown $T$, restart with fresh samples over geometric epochs of planned
length $H_j=2^{j+1}$, as in the doubling construction of
\citet[Theorem 4]{BessonKaufmann2018} and
\citet[Remark 4.7]{FouratiFramework2024}. Use the tuning of
Theorem~\ref{thm:fourati} in every epoch, including short epochs.
The prefix bound \eqref{eq:prefix-regret} below controls a possibly
unfinished final epoch. This step is needed because $\alpha$-regret can be
negative: a full-epoch bound alone does not bound its prefixes. Both powers
in the two-term tuning sum geometrically, so the anytime version retains
the same regret rate up to logarithmic factors under the stated total-horizon
condition.
\end{remark}

\section{Algorithms and Exact-Oracle Guarantees}
\label{sec:algorithms}
This section specifies the exact algorithmic components from
\citet{Kulik2026} whose resilience we study. Their process, swap definitions,
and Poisson guarantee are recalled with attribution, without re-proving the
imported theorem. The augmentation and truncation arguments specify the
additional controlled-oracle implementation used here. The Poisson rate,
valid-swap conditions, and spiteful drop step remain unchanged.

\subsection{Multilinear Extension}

For $x\in[0,1]^U$, let $R_x$ include each element $i$ independently with
probability $x_i$. Define $F(x):=\E[f(R_x)]$, the standard multilinear
extension used in submodular maximization \citep{CCPV2011,Feldman2011}.
For a set $A$ and $t\in[0,1]$, write $t\one_A$ for its scaled indicator.
The quantity $F(t\one_A\vee\one_i)-F(t\one_A)$
is the marginal used by SGS-Poisson.

\subsection{Valid Swaps}

\begin{definition}[Valid swap, \citealp{Kulik2026}]
For a matroid of rank $r\ge1$ (equal to $k$ before contraction and to $k'$ after
contraction), let $A\in\I$ and $t\in(0,1]$. A random pair
$(I,J)\in(A\cup\{\bot\})\times U$ is a valid swap if:
\begin{enumerate}[label=(\roman*)]
\item if $I\neq J$ then $A-I+J\in\I$ almost surely;
\item if $J\in A$, then $I=J$;
\item $\Pr(I=i)=1/r$ for every $i\in A$;
\item $\Pr(J=j)\le1/r$ for every $j\in U$.
\end{enumerate}
\end{definition}

It is $\eta$-almost-above-average if, for an optimal base $O$ and for every
$A\in\I$ and $t\in(0,1]$,
\begin{equation}
\E\bigl[
F(t\one_A\vee\one_J)-F(t\one_A)
\bigr]
\ge
\frac1r
\left(
F(t\one_A\vee\one_O)-F(t\one_A)
\right)
-\frac{\eta}{r}.
\label{eq:valid-swap}
\end{equation}

\begin{algorithm}[ht]
\caption{SGS-Poisson (base algorithm; \citealp{Kulik2026})}
\label{alg:sgs}
\begin{algorithmic}[1]
\Require Matroid of rank $r\ge1$, starting time $\eps_0\in(0,1]$, valid swap procedure
\State $t\gets\eps_0$, $A\gets\varnothing$
\State Sample the next event time $\tau(t)$ of a Poisson process with rate $r/t$; $t\gets\tau(t)$
\While{$t<1$}
    \State $(I,J)\gets\textsc{Swap}(t,A)$
    \State $A\gets A-I+J$
    \If{$I=J$}
        \State With probability $t$, set $A\gets A-I$ \Comment{spiteful drop}
    \EndIf
    \State Sample the next event time $\tau(t)$; $t\gets\tau(t)$
\EndWhile
\State \Return $A$
\end{algorithmic}
\end{algorithm}

The process in Algorithm~\ref{alg:sgs} is exactly the SGS-Poisson process of
\citet{Kulik2026}. In Algorithm~\ref{alg:sgs}, $r$ denotes the rank of the matroid
on which the process is currently run; after contraction this is $k'\le k$. In
particular, the spiteful drop in Lines 6--8 is not a modification introduced here.
The clock can be simulated in integrated time: from the current $t$, draw
an independent $E\sim\operatorname{Exp}(1)$ and set the next event time to
$t\exp(E/r)$. Clock draws are independent of the swap and drop randomness.
The expected number of swap calls is $r\log(1/\eps_0)$.

\subsection{Poisson-Process Guarantee}

We use the following theorem of \citet{Kulik2026} (their Lemma 3.3) as a
black-box analytical guarantee for the unchanged Poisson process.

\begin{proposition}[SGS-Poisson with an almost-above-average swap; \citealp{Kulik2026}]
\label{prop:sgs}
Suppose Algorithm~\ref{alg:sgs} uses a right-continuous
$\eta$-almost-above-average valid swap. Then its output $A$ satisfies
\begin{equation}
\E[f(A)]
\ge
\begin{cases}
(1-\eps_0)e^{-1}\OPT+e^{-1}f(\varnothing)-\eta,
&f\text{ non-monotone},\\[1mm]
(1-\eps_0)(1-e^{-1})\OPT+e^{-1}f(\varnothing)-\eta,
&f\text{ monotone},
\end{cases}
\label{eq:sgs-guarantee}
\end{equation}
and the expected number of swap calls is $r\ln(1/\eps_0)$.
\end{proposition}

The exact SGS-Poisson value-oracle implementation requires an estimate of the
optimum and a bound on the maximum sum of residual marginals. A naive noisy
greedy construction is not valid for general non-monotone matroid
maximization. We therefore use Residual Random Greedy.

\subsection{Dummy Augmentation}
\label{sec:dummy}

Add a set $D$ of $k$ dummy elements and form the rank-$k$ augmentation
\[
U^+:=U\cup D,
\qquad
\mathcal I^+:=\{S\subseteq U^+: S\cap U\in\mathcal I,
\ |S|\le k\}.
\]
Equivalently, $M^+=(U^+,\mathcal I^+)$ is the rank-$k$ truncation of the
direct sum of $M$ and the free matroid on $D$. In particular,
$\operatorname{rank}(M^+)=k$, and every independent set of $M$ can be
completed to a base of $M^+$ by adding dummy elements. Extend the objective
and controlled oracle by
\[
f^+(S):=f(S\setminus D),\qquad
\widehat f^+(S):=\widehat f(S\setminus D).
\]
Hence every dummy element has identically zero true and estimated marginal,
$\widehat f^+$ is a $\xierr$-controlled oracle for $f^+$, and an optimal
solution of the original instance can be extended to an optimal base of $M^+$
without changing its value. From this point through the preprocessing and
Poisson stages, we run the algorithm on the augmented instance
$(M^+,U^+,f^+,\widehat f^+)$ and suppress the superscript $+$ on
$M,f,\widehat f$ for notation. Within residual-marginal and swap formulas,
$U$ and $\I$ likewise refer to the ground set and independent-set family of
the instance being processed. In query-domain and online statements they
refer to the original instance; $n,k,\OPT$ always retain their original values.
The augmented instance has rank $k$, optimum $\OPT$, and ground-set size
$n+k\le2n$. At the end, dummy
elements are discarded from every final or fallback output; this preserves
feasibility and objective value.

The augmentation makes the base-completion assumptions explicit, including
for the controlled oracle, which need not be monotone or submodular.

\begin{lemma}[Dummy completion]
\label{lem:dummy}
Let $Q\in\mathcal I^+$ with $|Q|=t\le k$, let $h:2^{U^+}\to\mathbb R$ be a set
function invariant under adding dummy elements, and put $w(u):=h(u\mid Q)$ for $u$ in the contraction $M^+/Q$, which has
rank $r_t=k-t$. Then:
\begin{enumerate}[label=(\alph*),leftmargin=2em]
\item every independent set of $M^+/Q$ can be completed to a base of $M^+/Q$;
      in particular a base exists and has exactly $r_t$ elements;
\item every maximum-weight base $Z$ of $M^+/Q$ under $w$ satisfies
\[
\sum_{u\in Z}w(u)
=
\max_{T\subseteq U^+\setminus Q:\,T\cup Q\in\mathcal I^+}\ \sum_{u\in T}w(u);
\]
\item $\max\{h(T\cup Q): T\subseteq U^+\setminus Q,\ T\cup Q\in\mathcal I^+\}$
      is attained by a set $T$ with $T\cup Q$ a base of $M^+$.
\end{enumerate}
\end{lemma}

\begin{proof}
(a) $Q$ contains at most $t$ dummies, so at least $k-t=r_t$ dummies are unused;
any independent set $T$ of $M^+/Q$ has $|T|\le r_t$ and can be padded with
$r_t-|T|$ dummies outside $Q\cup T$ (at least that many remain), which keeps $(T\cup Q)\cap U\in\I$ and
$|T\cup Q|\le k$.

(b) Let $T^\star$ attain the right-hand maximum; we may assume
$w(u)\ge 0$ for all $u\in T^\star$. By (a), $T^\star$ extends to a base $Z^\star$
of $M^+/Q$ using unused dummies, all of which have $w=h(u\mid Q)=0$ because
$h$ ignores dummies. Hence $\sum_{u\in Z^\star}w(u)=\sum_{u\in T^\star}w(u)$, so
the maximum-weight base value is at least the right-hand side; the reverse
inequality is immediate since a base is in particular a feasible $T$.

(c) Identical padding argument.
\end{proof}

Part (b) is exactly what licenses replacing ``maximum over feasible $T$ of the
sum of marginals'' by ``sum of marginals over a maximum-weight residual base''
in the preprocessing stopping rule; part (c) licenses taking the optimum of a
contracted instance to be a \emph{base}, as required by \eqref{eq:valid-swap}.
The same completion argument applies to both the true and controlled objectives.

\subsection{Residual Random Greedy}

Run Residual Random Greedy (RRG) \citep{BFS2014} for $r=\lceil k/2\rceil$
iterations on the augmented instance. At iteration $i$, let $S_{i-1}$ be the
current set and $r_i=k-i+1$ the residual rank. In the contraction
$M/S_{i-1}$, compute a maximum-weight base $B_i$ with weights
$\widehat w_i(u)=\widehat f(u\mid S_{i-1})$, and choose $u_i$ uniformly from
$B_i$.

\begin{algorithm}[ht]
\caption{Residual Random Greedy, truncated to $\lceil k/2\rceil$ iterations \citep{BFS2014}}
\label{alg:rrg}
\begin{algorithmic}[1]
\Require Augmented matroid $M$, value oracle $h$
\State $S\gets\varnothing$
\For{$i=1,\ldots,\lceil k/2\rceil$}
    \State Compute a maximum-weight base $B_i$ of $M/S$ with weights $h(u\mid S)$
    \State Draw $u$ uniformly from $B_i$
    \State $S\gets S\cup\{u\}$
\EndFor
\State \Return $S\setminus D$
\end{algorithmic}
\end{algorithm}

We stop after $\lceil k/2\rceil$ iterations because the RRG analysis already
gives the constant-factor guarantee needed here at that point.

\subsection{Advanced Preprocessing and Contraction}
\label{sec:advanced-preprocessing}
For an independent set $Q\in\I$, define the residual marginal mass
\begin{equation}
\Mar_f(Q):=
\max_{\substack{T\subseteq U\setminus Q\\T\cup Q\in\I}}
\sum_{j\in T} f(j\mid Q),
\label{eq:residual-mass}
\end{equation}
and the global maximum singleton-marginal mass
$\Mar(f,\mathcal I):=\Mar_f(\varnothing)$.
For a contracted instance $(g,M')$ we write
$\Mar(g,M'):=\Mar\bigl(g,\mathcal I(M')\bigr)$.
The exact SGS-Poisson value-oracle implementation of
\citet[Lemma 4.3 and Appendix C]{Kulik2026} first constructs a constant-factor
upper estimate $V$ of $\OPT$ and then performs Advanced Preprocessing.
We use its exchange recursion with a strict continuation inequality:
starting from $Q_0=\varnothing$, while
$\Mar_f(Q_t)>20V$, it chooses a maximum-weight residual base under the true
marginals and adds a uniformly random element of that base.  The resulting set
is denoted by $\bar S$. Unlike the weak continuation inequality in the
source pseudocode, this convention stops at equality, including when both
the residual mass and the certificate are zero. The drift analysis needs
only the threshold inequality before stopping and is unaffected.

By Lemma~\ref{lem:dummy}(a) the process has at most $k$ iterations, and by
Lemma~\ref{lem:dummy}(b) the sum of marginals over the chosen base equals
$\Mar_f(Q_t)$.  The exact-oracle analysis shows that the stopping rule supplies
the bounded-residual-marginal condition required by the subsequent swap process.
In Section~\ref{sec:robustness} we show that the same preprocessing rule, with
exact values replaced by the given controlled oracle, retains this certificate
up to an additive $O(k\xierr)$ term.  Importantly, this is an analysis of the
true objective along the controlled-oracle preprocessing trajectory,
not an application of submodularity to the surrogate.

After preprocessing, contract the matroid at $\bar S$:
\begin{equation}
M':=M/\bar S,
\qquad
 g(T):=f(T\cup\bar S),
\qquad
\widehat g(T):=\widehat f(T\cup\bar S).
\label{eq:contraction}
\end{equation}
The contracted matroid has rank $k'\le k$; the SGS-Poisson process is run with
its actual rank $k'$, and throughout the analysis we upper bound $k'$ by $k$.
Then $g$ is nonnegative and submodular on $M'$, $\widehat g$ is a
$\xierr$-controlled oracle for $g$, and
$|\widehat g(i\mid T)-g(i\mid T)|\le2\xierr$.
Within a contracted instance, $F$ denotes the multilinear extension of $g$.
By Lemma~\ref{lem:dummy}(c) we may and do take $O_g$, a maximizer of $g$ over
$\mathcal I(M')$, to be a \emph{base} of $M'$.

\subsection{Value-Oracle Implementation of the Swap Rule}
\label{sec:swap-implementation}
We write $N_s$ for the number of Monte Carlo samples per swap.
At time $t$, for current set $A\in\I(M')$ and candidate element $i$, the
SGS-Poisson implementation of \citet{Kulik2026} uses the multilinear marginal
$w_i=F(t\one_A\vee\one_i)-F(t\one_A)$.
It estimates $w_i$ by drawing independent sets
$R_1,\ldots,R_{N_s}\sim t\one_A$ and averaging
\begin{equation}
\widetilde w_i
=
\frac1{N_s}\sum_{\ell=1}^{N_s}
\bigl[g(R_\ell\cup\{i\})-g(R_\ell)\bigr].
\label{eq:swap-estimator}
\end{equation}
The implementation then chooses a maximum-weight base $Z$ of the contracted matroid under these
estimates, constructs a matroid exchange map, and samples the entering element
uniformly from that base. Base ties and exchange-map choices use fixed
deterministic orders.
For completeness, complete $A$ to a base $B$ and choose a basis-exchange
bijection $h:Z\to B$ fixing $Z\cap B$ \citep{Schrijver2003}, with $B-h(j)+j$ a base for every
$j\in Z$. Set $h_A(j)=h(j)$ when $h(j)\in A$, and $h_A(j)=\bot$ otherwise.
Then $A-h_A(j)+j$ is independent, every element of $A$ has exactly one
preimage, and $h_A$ fixes $Z\cap A$. Drawing $J$ uniformly from $Z$ and
setting $I=h_A(J)$ therefore satisfies all valid-swap conditions.
The concentration analysis of the exact algorithm
controls the \emph{base sum} of these estimated marginals, which is the quantity
needed for the almost-above-average condition. Section~\ref{sec:robustness}
shows that this same implementation remains valid under the controlled oracle,
with $g$ replaced by $\widehat g$ in \eqref{eq:swap-estimator}.

When $k=1$ the direct enumeration described in Section~\ref{sec:main-theorem}
gives a stronger guarantee; the swap analysis therefore uses $k\ge2$.
Right-continuity of the transition kernel is proved in
Lemma~\ref{lem:right-continuity}. If the contracted rank is zero, return
$\bar S\setminus D$ directly. Otherwise the Poisson process is run in the contracted
matroid and returns $S_{\rm out}\in\I(M/\bar S)$, so
$A_{\rm out}:=(S_{\rm out}\cup\bar S)\setminus D$ is feasible in the original matroid, and
\begin{equation}
f(A_{\rm out})=g(S_{\rm out}).
\label{eq:output-identity}
\end{equation}

\subsection{A Deterministic Query Budget}
\label{sec:budget}

The number of Poisson events is random. As explained in
Remark~\ref{rem:deterministic-N}, the offline-to-online reduction needs a
deterministic bound. We truncate.

\begin{lemma}[Truncation]
\label{lem:budget}
Fix $\eps_0\in(0,1]$ and $\rho\in(0,1)$, and let $\Lambda:=k\ln(1/\eps_0)$ be an upper bound on the
mean number of Poisson events. Let
$N_{\max}:=\lceil\max\{2e\Lambda,\ \log_2(1/\rho)\}\rceil$
and modify Algorithm~\ref{alg:sgs} so that if the number of events reaches
$N_{\max}$ the process is aborted and the empty set is returned in the
contracted instance (so that the overall output is $\bar S\setminus D$). Let
$A_{\rm out}$ be the output of the truncated implementation and $A^{\rm full}$
that of the untruncated one, coupled on the same randomness. Then the number of
swap calls is at most $N_{\max}=O(1+k\log(1/\eps_0)+\log(1/\rho))$ deterministically, and
\begin{equation}
\E[f(A_{\rm out})]\ \ge\ \E[f(A^{\rm full})]-\rho\,\OPT .
\label{eq:truncation-loss}
\end{equation}
\end{lemma}

\begin{proof}
Let $N_P$ be the number of events of a Poisson process with mean
$\Lambda'\le\Lambda$. The standard Chernoff bound
$\Prb(N_P\ge x)\le e^{-\Lambda'}(e\Lambda'/x)^{x}$ gives
$\Prb(N_P\ge x)\le 2^{-x}$ whenever $x\ge 2e\Lambda'$. With
$x=N_{\max}\ge\max\{2e\Lambda,\log_2(1/\rho)\}$ we get
$\Prb(N_P\ge N_{\max})\le 2^{-N_{\max}}\le\rho$.
Let $G$ be the event that truncation does not occur, so $\Prb(G^c)\le\rho$ and
$A_{\rm out}=A^{\rm full}$ on $G$. Since $A^{\rm full}$ is feasible,
$f(A^{\rm full})\le\OPT$, and since $f(\bar S)\ge0$,
\[
\E[f(A_{\rm out})]
\ge\E[f(A^{\rm full})\one_G]
=\E[f(A^{\rm full})]-\E[f(A^{\rm full})\one_{G^c}]
\ge\E[f(A^{\rm full})]-\rho\OPT.
\]
Finally, the Chernoff bound above holds conditionally on the contracted rank
$k'\le k$ and on all randomness generated before the Poisson stage, because the
Poisson clock is independent of that history and its mean is at most $\Lambda$
for every realization of $k'$. Hence \eqref{eq:truncation-loss} also holds
conditionally on any event measurable with respect to that history; this
conditional form is what the proof of Theorem~\ref{thm:main} uses.
\end{proof}

The bound \eqref{eq:truncation-loss} is relative to $\OPT$, not absolute; this
is essential, since Definition~\ref{def:resilience} demands a guarantee of the
form $(\alpha-\eps)\OPT-\delta\xierr$ and an absolute $-\rho$ loss would
destroy it when $\OPT$ is small.

\subsection{Every Queried Set is Almost Feasible}
\label{sec:feasible-queries}

In the offline value-oracle model the algorithm may query $f$ at arbitrary
subsets. In the CMAB application a query is answered by \emph{playing} the
queried super-arm, so queried sets must be admissible actions. Define
\begin{equation}
\Ipl:=\{S\subseteq U:\ S\in\I\ \text{ or }\ S=S'\cup\{i\}\ \text{for some }S'\in\I,\ i\in U\}.
\label{eq:Iplus}
\end{equation}

\begin{lemma}[Query feasibility]
\label{lem:query-feasibility}
Every set at which the implementation of Sections~\ref{sec:dummy}--\ref{sec:budget}
queries the oracle belongs to $\Ipl$ after dummy elements are deleted.
\end{lemma}

\begin{proof}
Queries on the augmented instance are of the form $\widehat f^+(S)=\widehat f(S\setminus D)$,
so it suffices to check $S\setminus D$ for each query type, and $S\cap U\in\I$
whenever $S\in\mathcal I^+$.
RRG queries $S_{i-1}$ and $S_{i-1}+u$ with $S_{i-1}+u\in\mathcal I^+$; Advanced
Preprocessing queries $Q_t$ and $Q_t+j$ with $Q_t+j\in\mathcal I^+$; the $k=1$
enumeration queries $\varnothing$ and feasible singletons. The optimum
certificate also evaluates the feasible RRG outputs. All of these are in
$\I$ after deleting dummies. The swap estimator queries
$\widehat g(R_\ell)=\widehat f(\bar S\cup R_\ell)$ with
$\bar S\cup R_\ell\in\mathcal I^+$, hence in $\I$ after deleting dummies, and
$\widehat g(R_\ell\cup\{i\})=\widehat f(\bar S\cup R_\ell\cup\{i\})$, which is a
feasible set plus one element, hence in $\Ipl$.
\end{proof}

The last family is the only one that can leave $\I$. If $i\notin A$ and
$A+i$ contains a circuit $C$, then $R_\ell+i$ is dependent exactly when
$C\setminus\{i\}\subseteq R_\ell$, an event of probability
$t^{|C|-1}$. For example, under a rank-$k$ uniform matroid with $n>k$,
a current base $A$ and $i\notin A$ yield the infeasible query $A+i$ when
$R_\ell=A$, with probability $t^k$. Section~\ref{sec:cmab} explicitly
separates the enlarged exploration set from the feasible benchmark.

\section{Robustness of the Main Algorithm}
\label{sec:robustness}
We now analyze the algorithm of Section~\ref{sec:algorithms} under the controlled
oracle of Section~\ref{sec:setup}. The SGS-Poisson dynamics, valid-swap rule,
and spiteful drop are inherited from \citet{Kulik2026}; preprocessing uses
the exchange recursion with the stopping convention of
Section~\ref{sec:advanced-preprocessing}. RRG \citep{BFS2014} supplies the
error-aware certificate, and subsequent comparisons use the available
oracle. We prove that the approximation survives these perturbations and
the implementation's budget truncations. Until
Section~\ref{sec:main-theorem}, assume $k\ge2$; that section separately
handles ranks zero and one before stating the general theorem.

\subsection{Robust Residual Random Greedy and the Optimum Certificate}
\label{sec:robust-rrg}
We first show that the same Residual Random Greedy procedure remains a
constant-factor procedure when its value queries are answered by $\widehat f$.
The RRG exchange coupling is unchanged; only the maximum-base comparison is
affected by the oracle perturbation.

\begin{lemma}[Robust maximum-base comparison]
\label{lem:base}
At iteration $i$, for every base $B$ of the residual matroid $M/S_{i-1}$,
\begin{equation}
\sum_{u\in B_i}f(u\mid S_{i-1})
\ge
\sum_{u\in B}f(u\mid S_{i-1})-2r_i\xierr.
\label{eq:base-comparison}
\end{equation}
\end{lemma}

\begin{proof}
By optimality of $B_i$ under $\widehat f$,
$\sum_{u\in B_i}\widehat f(u\mid S_{i-1})\ge\sum_{u\in B}\widehat f(u\mid S_{i-1})$.
Both bases have exactly $r_i$ elements by Lemma~\ref{lem:dummy}(a), so
Lemma~\ref{lem:common-base} with $X=B_i$, $Y=B$, $S=S_{i-1}$ applies and yields
\eqref{eq:base-comparison}.
\end{proof}

We recall the standard RRG exchange coupling (Observation 4.2 in the full
version of \citealp{BFS2014}), using the basis-exchange bijection lemma
\citep[Corollary 39.12a]{Schrijver2003}. The sampling fact used
both here and in preprocessing is \citet[Lemma 2.3]{Kulik2026}: for a
nonnegative submodular function $h$, if $p,q\in[0,1]$ and a random set $R$
satisfies $\Prb(a\in R)\ge p$ for $a\in B$ and $\Prb(a\in R)\le q$ for
$a\notin B$, then
\begin{equation}
\E[h(R)]\ge(p-q)h(B).
\label{eq:rrg-sampling-lemma}
\end{equation}
No independence between element inclusions is required. This is a standard
consequence of the Lov\'asz-extension bound \citep{Lovasz1983}; we use the
result without reproducing its proof.

\begin{lemma}[Standard RRG exchange coupling]
\label{lem:rrg-exchange}
Consider the augmented matroid $M^+$ and let $O^+$ be an optimal base
containing an optimal solution of the original instance (together with
dummy elements). Assume $k\ge2$. The RRG coupling can be chosen so that, for
$0\le i\le\lceil k/2\rceil$,
\begin{equation}
\mathbb E[f^+(S_i\cup O_i)]
\ge
\frac{(k-i)(k-i-1)}{k(k-1)}\OPT ,
\label{eq:rrg-exchange}
\end{equation}
where $O_0=O^+$, $O_i\subseteq O_{i-1}$, $S_i\cap O_i=\varnothing$, and $S_i\cup O_i$ is a base of
$M^+$ almost surely.
\end{lemma}

The cited exchange argument depends only on uniform sampling from a
residual base, not on optimality of that base under the true marginals.
Consequently the coupling is unchanged when the base is selected using
$\widehat f$. The controlled-oracle loss enters only through
Lemma~\ref{lem:base}, which is the additional ingredient in the next proof.

\begin{lemma}[Robust RRG]
\label{lem:rrg}
Assume $k\ge2$. For $1\le i\le\lceil k/2\rceil$,
\begin{equation}
\mathbb E[f^+(S_i)]
\ge
\frac{i(k-i)}{k(k-1)}\OPT-2i\xierr.
\label{eq:rrg-guarantee}
\end{equation}
In particular, for $r=\lceil k/2\rceil$ and $G:=S_r\setminus D\in\I$,
\begin{equation}
\mathbb E[f(G)]
\ge\frac14\OPT-2k\xierr.
\label{eq:rrg-half}
\end{equation}
\end{lemma}

\begin{proof}
The recurrence follows the proof of Theorem 1.5 in the full version of
\citet{BFS2014}. The point to
check is that the controlled comparison loses only $2r_i\xierr$ before
averaging over the residual base, so the accumulated error is $O(k\xierr)$.
Let $r_i=k-i+1$.  Conditional on the history through iteration $i-1$,
Lemma~\ref{lem:base} with $B=O_{i-1}$ and submodularity give
\[
\begin{aligned}
\mathbb E[f^+(S_i)-f^+(S_{i-1})\mid\mathcal F_{i-1}]
&=
\frac1{r_i}\sum_{u\in B_i}f^+(u\mid S_{i-1})\\
&\ge
\frac{f^+(S_{i-1}\cup O_{i-1})-f^+(S_{i-1})}{r_i}
-2\xierr .
\end{aligned}
\]
Taking expectations and applying Lemma~\ref{lem:rrg-exchange} at $i-1$,
\[
\mathbb E[f^+(S_i)]
\ge
\frac{k-i}{k-i+1}\mathbb E[f^+(S_{i-1})]
+
\frac{k-i}{k(k-1)}\OPT
-2\xierr.
\]
An induction on $i$, starting from $S_0=\varnothing$, gives
\eqref{eq:rrg-guarantee}.  Finally, $\frac{r(k-r)}{k(k-1)}\ge\frac14$
for $r=\lceil k/2\rceil$ and $k\ge2$, $r\le k$, and $f^+(S_r)=f(S_r\setminus D)$.
\end{proof}

For $k=1$, querying $\varnothing$ and all feasible singletons gives a stronger
$\OPT-2\xierr$ estimate, so the remainder assumes $k\ge2$.

\subsection{High-Probability Optimum Certificate}

The expectation guarantee in Lemma~\ref{lem:rrg} is not enough because
Advanced Preprocessing uses a \emph{realized} threshold. Fix $\rho\in(0,1)$
and amplify by running Algorithm~\ref{alg:rrg} independently
$
R=\left\lceil 14\ln(1/\rho)\right\rceil
$
times, and let $G^\star$ be the returned set with the largest $\widehat f$ value.
Define
\begin{equation}
\widehat V
:=
\max\left\{
32k\xierr,\;
8\widehat f(G^\star)+16\xierr
\right\}.
\label{eq:certificate-estimator}
\end{equation}

\begin{lemma}[Optimum certificate]
\label{lem:certificate}
Deterministically, $\widehat V\le8\OPT+32k\xierr$. Moreover
$\Prb(\widehat V\ge\OPT)\ge1-\rho$, so that
\begin{equation}
\OPT\le\widehat V\le8\OPT+32k\xierr
\label{eq:certificate-bound}
\end{equation}
with probability at least $1-\rho$.
\end{lemma}

\begin{proof}
\emph{Upper bound.} Since $G^\star\in\I$ we have
$\widehat f(G^\star)\le f(G^\star)+\xierr\le\OPT+\xierr$, so
$8\widehat f(G^\star)+16\xierr\le8\OPT+24\xierr$. As $k\ge1$,
$\max\{32k\xierr,\,8\OPT+24\xierr\}\le8\OPT+32k\xierr$.

\emph{Lower bound.} If $\OPT\le32k\xierr$ the first term of
\eqref{eq:certificate-estimator} already gives $\widehat V\ge\OPT$
deterministically; this covers $\OPT=0$. So assume $\OPT>32k\xierr$. Then
$2k\xierr\le\OPT/16$ and \eqref{eq:rrg-half} gives
$\E[f(G)]\ge\frac14\OPT-\frac1{16}\OPT=\frac3{16}\OPT$ for a single run.
Let $p:=\Prb(f(G)\ge\OPT/8)$. Since $0\le f(G)\le\OPT$,
$\frac3{16}\OPT\le p\OPT+(1-p)\frac18\OPT$, i.e.\ $p\ge\frac1{14}$.
Hence with probability at least $1-(1-\frac1{14})^{R}\ge1-e^{-R/14}\ge1-\rho$
some run $j$ has $f(G^{(j)})\ge\OPT/8$, and then
$\widehat f(G^\star)\ge\widehat f(G^{(j)})\ge f(G^{(j)})-\xierr\ge\OPT/8-\xierr$,
so $8\widehat f(G^\star)+16\xierr\ge\OPT+8\xierr\ge\OPT$.
\end{proof}

\subsection{Robust Adaptive Preprocessing}
\begin{lemma}[Marginal-mass perturbation]
\label{lem:mar}
For every $Q\in\I$, $|\Mar_{\widehat f}(Q)-\Mar_f(Q)|\le2k\xierr$.
\end{lemma}

\begin{proof}
Every feasible $T$ in \eqref{eq:residual-mass} has at most $k$ elements, and the
difference between its estimated and true marginal sum is at most $2k\xierr$ by
\eqref{eq:oracle-marginal}. Taking maxima in both directions proves the result.
\end{proof}

Starting from $Q_0=\varnothing$, define the stopping time
\begin{equation}
\tau:=\inf\{t:\Mar_{\widehat f}(Q_t)\le20\widehat V\}.
\label{eq:stopping-time}
\end{equation}
For $t<\tau$, choose a maximum-weight residual base $Z_t$ of $M/Q_t$ under the
estimated marginals $\widehat f(\cdot\mid Q_t)$ and draw $j_t$ uniformly from
$Z_t$; set $Q_{t+1}=Q_t+j_t$. By Lemma~\ref{lem:dummy}(a) the residual matroid
always has a base, so $\tau\le k$; and if $t=k$ then the residual rank is $0$
and $\Mar_{\widehat f}(Q_k)=0\le20\widehat V$, so the process indeed stops.
We write $\bar S:=Q_\tau$.

\begin{lemma}[Residual marginal bound]
\label{lem:residual}
Deterministically,
\begin{equation}
\Mar_f(\bar S)
\le
160\OPT+642\,k\xierr .
\label{eq:residual-bound}
\end{equation}
\end{lemma}

\begin{proof}
At stopping, $\Mar_{\widehat f}(\bar S)\le20\widehat V$. By Lemma~\ref{lem:mar}
and the deterministic upper bound in Lemma~\ref{lem:certificate},
$\Mar_f(\bar S)\le20(8\OPT+32k\xierr)+2k\xierr$.
\end{proof}

The crucial result is that the preprocessing certificate itself is resilient.
Let $O_0=O$ be an optimal base of the augmented matroid. By basis exchange,
construct $O_t\subseteq O$ so that $Q_t\cap O_t=\varnothing$ and
$Q_t\cup O_t$ is a base at every time. Let $\mathcal F_t$ contain the
initial RRG randomness and all preprocessing choices before step $t$. Define
the exchange potential
\begin{equation}
\Phi_t:=f(Q_t\cup O_t)+\tfrac12f(Q_t).
\label{eq:potential}
\end{equation}

\begin{lemma}[Resilient adaptive-preprocessing drift]
\label{lem:drift}
On the event $\{\widehat V\ge\OPT\}$, for every $t<\tau$,
\begin{equation}
\mathbb E[\Phi_{t+1}-\Phi_t\mid\mathcal F_t]
\ge
\frac{8\OPT-k\xierr}{k-t}.
\label{eq:drift}
\end{equation}
\end{lemma}

\begin{proof}
We adapt the exchange-potential argument of \citet[Appendix C]{Kulik2026}.
The additional step is to justify the drift under a base and a stopping
decision chosen using $\widehat f$, while every submodularity inequality
below is applied to $f$.
Fix $t<\tau$ and write $r_t=k-t$.  Conditional on $\mathcal F_t$, the
current set $Q_t$, the estimated marginal weights, and the selected maximum
weight base $Z_t$ are fixed.  Because the stopping condition has not fired,
and by Lemma~\ref{lem:dummy}(b),
\[
\sum_{j\in Z_t}\widehat f(j\mid Q_t)
=\Mar_{\widehat f}(Q_t)
>
20\widehat V
\ge20\OPT .
\]
The base $Z_t$ has $r_t$ elements, so by \eqref{eq:oracle-marginal},
$\sum_{j\in Z_t}f(j\mid Q_t)\ge20\OPT-2r_t\xierr$.
Since $j_t$ is uniform in $Z_t$,
\begin{equation}
\mathbb E[f(Q_{t+1})-f(Q_t)\mid\mathcal F_t]
\ge
\frac{20\OPT-2r_t\xierr}{r_t}.
\label{eq:preprocess-certificate}
\end{equation}

We next make the exchange argument explicit.  Let $O_t$ be the current
exchange set, so $Q_t\cup O_t$ is a base.  By the basis-exchange bijection
lemma there is a bijection $h_t:Z_t\to O_t$ such that
$Q_t\cup O_t-h_t(j)+j\in\mathcal I$ for all $j\in Z_t$, with $h_t(j)=j$ on
$Z_t\cap O_t$.  After drawing $j_t$ uniformly from $Z_t$, set
$O_{t+1}:=O_t\setminus\{h_t(j_t)\}$; then $Q_{t+1}\cup O_{t+1}$ is a base.
Every element of $Q_t\cup O_t$ is removed from the next base with probability at
most $1/r_t$, while every element outside $Q_t\cup O_t$ is inserted with
probability at most $1/r_t$.  Applying \eqref{eq:rrg-sampling-lemma} with
$p=1-1/r_t$ and $q=1/r_t$ gives
\begin{equation}
\mathbb E[
f(Q_{t+1}\cup O_{t+1})\mid\mathcal F_t]
\ge
\left(1-\frac{2}{r_t}\right)f(Q_t\cup O_t).
\label{eq:contracted-instance}
\end{equation}
For $r_t=1$ the right-hand side is nonpositive, so the same inequality
holds trivially by nonnegativity.
Combining \eqref{eq:preprocess-certificate} and \eqref{eq:contracted-instance},
\[
\mathbb E[\Phi_{t+1}-\Phi_t\mid\mathcal F_t]
\ge
\frac{10\OPT-r_t\xierr-2f(Q_t\cup O_t)}{r_t}
\ge
\frac{8\OPT-r_t\xierr}{r_t}
\ge
\frac{8\OPT-k\xierr}{k-t},
\]
where we used $f(Q_t\cup O_t)\le\OPT$ by feasibility.
\end{proof}

\begin{theorem}[Resilient preprocessing certificate]
\label{thm:preprocess}
Let $\mathcal E$ be any event measurable with respect to the initial
RRG randomness with $\Prb(\mathcal E)>0$ such that
$\mathcal E\subseteq\{\widehat V\ge\OPT\}$.  Then the
output $\bar S$ of Advanced Preprocessing satisfies
\begin{equation}
\mathbb E\left[
\max_{T:\,T\cup\bar S\in\I}
f(T\cup\bar S)
+\tfrac12f(\bar S)
\;\middle|\;\mathcal E
\right]
\ge
\OPT-\tfrac18k\xierr .
\label{eq:preprocess-contracted-certificate}
\end{equation}
\end{theorem}

\begin{proof}
Condition on $\mathcal E$ and set $\widetilde\Phi_t:=\Phi_{t\wedge\tau}$.
If $\OPT\ge k\xierr/8$, then Lemma~\ref{lem:drift} shows that
$(\widetilde\Phi_t)_{t\ge0}$ is a submartingale under the conditional
probability given $\mathcal E$. Since $\tau\le k$ and $0\le\Phi_t\le3/2$, the
stopped submartingale is bounded, so the bounded optional-stopping theorem gives
\[
\mathbb E[\Phi_\tau\mid\mathcal E]
\ge
\Phi_0
=
f(O)+\tfrac12 f(\varnothing)
\ge\OPT.
\]
If $\OPT<k\xierr/8$, nonnegativity gives
$\mathbb E[\Phi_\tau\mid\mathcal E]\ge0\ge\OPT-\frac18k\xierr$.
Thus in both cases $\mathbb E[\Phi_\tau\mid\mathcal E]\ge\OPT-\frac18k\xierr$.
Since $Q_\tau=\bar S$ and $O_\tau$ is feasible in the contraction,
$f(\bar S\cup O_\tau)\le\max_{T:T\cup\bar S\in\I}f(T\cup\bar S)$, and
$\Phi_\tau=f(\bar S\cup O_\tau)+\frac12f(\bar S)$, so the claimed certificate
follows.
\end{proof}

\paragraph{Why this lemma is the technical centerpiece.}
The proof does not compare the controlled-oracle trajectory to the exact
trajectory. The two trajectories may diverge completely. Instead, the potential
is shown to retain nonnegative drift in the nontrivial regime
$\OPT\ge k\xierr/8$, \emph{on the controlled-oracle trajectory itself}. This is the reason the result is not a routine Lipschitz perturbation
argument.

\subsection{Robust Almost-Above-Average Swaps}
\label{sec:robust-swap}
Recall the contraction \eqref{eq:contraction}. By Theorem~\ref{thm:preprocess},
on the event $\mathcal E$,
\begin{equation}
\E\left[g(O_g)+\tfrac12g(\varnothing)\;\middle|\;\mathcal E\right]
\ge
\OPT-\tfrac18k\xierr,
\label{eq:contracted-certificate}
\end{equation}
and by Lemma~\ref{lem:residual}, deterministically,
\begin{equation}
\Mar(g,M')=\Mar_f(\bar S)
\le
160\OPT+642\,k\xierr.
\label{eq:residual-contracted}
\end{equation}

\begin{lemma}[Right-continuity of the controlled swap]
\label{lem:right-continuity}
On a contracted instance of rank $k'\ge1$, the controlled-oracle
implementation of the SGS-Poisson swap is right-continuous in the sense required by Proposition~\ref{prop:sgs}.
\end{lemma}

\begin{proof}
Fix a state $A$ and the sample size $N_s$. For each possible tuple
$(R_1,\ldots,R_{N_s})$ of subsets of $A$, its sampling probability is
\[
\prod_{\ell=1}^{N_s}t^{|R_\ell|}(1-t)^{|A|-|R_\ell|},
\]
a polynomial in $t$. Given this tuple, the persistent oracle fixes all
estimated weights, and deterministic tie-breaking fixes the selected base
and exchange map independently of $t$. The conditional probability of each
swap is therefore constant for that tuple. Consequently $p_{ij}(t,A)$ is a
finite sum of polynomials and is continuous, hence right-continuous.
This argument concerns the transition kernel; fresh samples are drawn at
each actual Poisson event. It requires no submodularity or continuity of
$\widehat g$.
\end{proof}

\begin{lemma}[Robust swap concentration]
\label{lem:swap}
Fix a contracted instance of rank $k'\ge1$ and $\delta_s\in(0,1/2)$.
Use the same value-oracle sampling,
maximum-weight-base selection, exchange-map construction, and fixed
tie-breaking convention as the exact SGS-Poisson implementation of
\citet{Kulik2026}, but evaluate every set through the controlled oracle
$\widehat g$.  If
\begin{equation}
N_s=\left\lceil
\frac{25}{2\delta_s^2}
\left(k\ln(2en)+\ln\frac{10}{\delta_s}\right)
\right\rceil
=O\!\left(\frac{k\log n+\log(1/\delta_s)}{\delta_s^2}\right),
\label{eq:swap-samples}
\end{equation}
then the resulting swap is valid and right-continuous and satisfies
\begin{equation}
\eta
\le
\tfrac45\,\delta_s L
+4k\xierr,
\qquad
L:=\Mar(g,M')+g(O_g).
\label{eq:swap-error}
\end{equation}
\end{lemma}

\begin{proof}
The proof follows the base-sum concentration method of
\citet[Lemma 4.5]{Kulik2026}. We retain the details needed to separate true
sampling error from controlled-oracle bias: the uniform event concerns the
true objective, while the selected base is chosen using perturbed weights.
This separation is what avoids assuming submodularity of $\widehat g$.
Validity is purely combinatorial.  Once a maximum-weight base $Z$ has been
selected, the deterministic exchange map used by the exact implementation
satisfies the valid-swap conditions.  Right-continuity follows from
Lemma~\ref{lem:right-continuity}.

Fix a swap time and condition on the complete history $\mathcal H$ immediately
before the fresh samples are drawn; thus $A$, $t$ and the contracted matroid are
fixed.  For every independent set $Z$ of the contracted matroid, define
\[
X_\ell(Z):=
\sum_{i\in Z}
\bigl(g(R_\ell\cup\{i\})-g(R_\ell)\bigr),
\qquad R_\ell\sim t\one_A\ \text{i.i.d.},
\]
with population expectation
\[
W_g(Z):=\mathbb E[X_\ell(Z)\mid\mathcal H]
=
\sum_{i\in Z}
\left(F(t\one_A\vee\one_i)-F(t\one_A)\right),
\]
and empirical \emph{true} base sum
$\overline W_g(Z):=\frac1{N_s}\sum_{\ell=1}^{N_s} X_\ell(Z)$.

\emph{Step 1: range.} Let $Z^+(R):=\{i\in Z\setminus R:g(i\mid R)\ge0\}$.
Terms with $i\in R$ vanish. For $i\notin R$, diminishing returns gives
$g(i\mid R)\le g(i\mid\varnothing)$, and $Z^+(R)$ is independent. Thus
\[
X_\ell(Z)
\le
\sum_{i\in Z^+(R_\ell)}g(i\mid R_\ell)
\le
\Mar(g,M').
\]
For the lower bound, submodularity gives
$X_\ell(Z)\ge g(R_\ell\cup Z)-g(R_\ell)\ge-g(R_\ell)\ge-g(O_g)$,
because $R_\ell\subseteq A$ is independent and $g(O_g)$ is the optimum
value of the contracted instance.  Hence
\begin{equation}
-g(O_g)\le X_\ell(Z)\le\Mar(g,M')
\qquad\forall Z,\ell,
\label{eq:true-base-sum-range}
\end{equation}
so $X_\ell(Z)$ has range at most $L$.
If $L=0$ then $g(O_g)=0$, the contracted instance has zero optimum, and the
almost-above-average condition is immediate; so assume $L>0$.

\emph{Step 2: uniform concentration.} Hoeffding's inequality
\citep{Hoeffding1963} and \eqref{eq:true-base-sum-range} give, for every fixed
independent set $Z$,
\[
\Prb\left(
\left|
\overline W_g(Z)-W_g(Z)
\right|
>\frac{\delta_s}{5}L
\,\middle|\,\mathcal H
\right)
\le
2\exp\left(-\frac{2N_s\delta_s^2}{25}\right).
\]
The contracted augmented matroid has rank at most $k$ and ground-set size at
most $2n$, so its number of independent sets is at most
$\sum_{j=0}^k\binom{2n}{j}\le(2en)^k$,
whose logarithm is $O(k\log n)$. Thus the explicit sample size in \eqref{eq:swap-samples} ensures that the event
\[
\mathcal E_{\rm sw}:=
\left\{
\sup_{Z\in\mathcal I(M')}
\left|
\overline W_g(Z)-W_g(Z)
\right|
\le\frac{\delta_s}{5}L
\right\}
\]
satisfies
$\Prb(\mathcal E_{\rm sw}^c\mid\mathcal H)\le\delta_s/5$.
Note that $\mathcal E_{\rm sw}$ is defined \emph{before} the data-dependent base
is selected, which is what makes the argument valid for an adaptively chosen
comparison.

\emph{Step 3: oracle bias.} For every sampled set $R$ and candidate element $i$,
$\bigl|[\widehat g(R\cup\{i\})-\widehat g(R)]-[g(R\cup\{i\})-g(R)]\bigr|\le2\xierr$,
hence for every independent set $Z$,
\begin{equation}
\Bigl|
\textstyle\sum_{i\in Z}
[\widehat g(R\cup\{i\})-\widehat g(R)]
-
\sum_{i\in Z}
[g(R\cup\{i\})-g(R)]
\Bigr|
\le2k\xierr,
\label{eq:swap-oracle-bias}
\end{equation}
deterministically and simultaneously for all $Z$. Writing $\widehat W(Z)$ for
the empirical base sum computed from the controlled oracle, on
$\mathcal E_{\rm sw}$,
\begin{equation}
\sup_{Z\in\mathcal I(M')}
|\widehat W(Z)-W_g(Z)|
\le
\Delta,
\qquad
\Delta:=\frac{\delta_s}{5}L+2k\xierr.
\label{eq:swap-uniform-controlled}
\end{equation}

\emph{Step 4: comparison against the optimal base.} Let $\widehat Z$ be the
data-dependent maximum-weight base selected by the algorithm; $O_g$ is an
admissible comparison base by Lemma~\ref{lem:dummy}(c). Because
\eqref{eq:swap-uniform-controlled} holds simultaneously for every independent
set, it applies to $\widehat Z$ and to $O_g$, even though $\widehat Z$ depends
on the same samples. Therefore, on $\mathcal E_{\rm sw}$,
\[
\begin{aligned}
W_g(\widehat Z)
&\ge \widehat W(\widehat Z)-\Delta
&&\text{by \eqref{eq:swap-uniform-controlled}}\\
&\ge \widehat W(O_g)-\Delta
&&\text{by optimality of $\widehat Z$ under $\widehat W$}\\
&\ge W_g(O_g)-2\Delta
&&\text{by \eqref{eq:swap-uniform-controlled}}\\
&\ge \Gamma-2\Delta,
&&
\end{aligned}
\]
where $\Gamma:=F(t\one_A\vee\one_{O_g})-F(t\one_A)$ and the last step is the standard
submodularity bound
$W_g(O_g)=\sum_{i\in O_g}\bigl(F(t\one_A\vee\one_i)-F(t\one_A)\bigr)\ge \Gamma$.

\emph{Step 5: averaging.} Conditional on the realized samples, the entering
element is uniform in $\widehat Z$, so its conditional expected gain is
$W_g(\widehat Z)/k'$ where $k'\le k$ is the contracted rank. The empirical
quantities are used only to lower bound $W_g(\widehat Z)$. Hence
\[
\mathbb E\bigl[
F(t\one_A\vee\one_J)-F(t\one_A)
\mid\mathcal H,\mathcal E_{\rm sw},R_1,\ldots,R_{N_s}\bigr]
\ge
\frac{\Gamma-2\Delta}{k'} .
\]
On the failure event, \eqref{eq:true-base-sum-range} gives
$W_g(Z)\ge-g(O_g)$ for every $Z$, hence the same conditional expectation is at
least $-g(O_g)/k'$. Let $p:=\Prb(\mathcal E_{\rm sw}^c\mid\mathcal H)\le\delta_s/5$.
Averaging over the fresh samples,
\[
\mathbb E\bigl[
F(t\one_A\vee\one_J)-F(t\one_A)\mid\mathcal H\bigr]
\ge
\frac{(1-p)(\Gamma-2\Delta)-p\,g(O_g)}{k'}
\ge
\frac{\Gamma-2\Delta-pL}{k'},
\]
because $\Gamma\le W_g(O_g)\le\Mar(g,M')$ by submodularity and
\eqref{eq:true-base-sum-range}, so $\Gamma+g(O_g)\le L$, and $\Delta\ge0$.
This argument uses the residual-mass bound rather than feasibility of
$R_\ell\cup O_g$, which need not hold. Comparing with \eqref{eq:valid-swap} gives
\[
\eta\le2\Delta+pL
\le\frac35\delta_sL+4k\xierr
\le\frac45\delta_sL+4k\xierr,
\]
which is \eqref{eq:swap-error}. No submodularity, monotonicity, or unbiasedness
of $\widehat g$ is used.
\end{proof}

Combining \eqref{eq:residual-contracted}, $g(O_g)\le\OPT$, and
Lemma~\ref{lem:swap}, we obtain
$\eta\le\frac45\delta_s(161\OPT+642k\xierr)+4k\xierr$. Setting
\begin{equation}
\delta_s=\frac{\eps}{250}
\label{eq:delta-s}
\end{equation}
and using $\eps\le1/2$ gives the explicit bound
\begin{equation}
\eta\le \frac{13}{25}\,\eps\,\OPT+6k\xierr .
\label{eq:eta-bound}
\end{equation}
Indeed $\tfrac45\cdot\tfrac{\eps}{250}\cdot161=0.5152\,\eps\le\tfrac{13}{25}\eps$
and $\tfrac45\cdot\tfrac{\eps}{250}\cdot642+4\le 2.06\cdot\tfrac12+4<6$.

Notice the important \emph{relative} form of the first term. We do not replace
it by $O(\eps)$: doing so would lose the desired $(\alpha-\eps)\OPT$
resilience statement when $\OPT$ is small.

\begin{remark}[Explicit constants]
The chain above keeps every constant explicit. Two choices matter for the final
oracle complexity: the amplification threshold in
\eqref{eq:certificate-estimator}, which controls the factor $8$ in
$\widehat V\le8\OPT+\ldots$ and hence the factor $161$ in $L$; and the resulting
choice \eqref{eq:delta-s} of $\delta_s$. The sampling budget in
\eqref{eq:swap-samples} depends on
$\delta_s^{-2}\bigl(k\log n+\log(1/\delta_s)\bigr)$, so decreasing
$\delta_s$ increases the required sample budget. The choice
$\delta_s=\eps/250$ yields the error bound \eqref{eq:eta-bound} while
preserving the stated asymptotic oracle complexity.
\end{remark}

\subsection{Oracle Complexity}

\begin{lemma}[Controlled-oracle complexity]
\label{lem:oracle-complexity}
With $\rho=\eps/100$, $\delta_s=\eps/250$ and
$\eps_0=\eps/100$, the number of controlled-oracle evaluations
used by the truncated SGS-Poisson value-oracle implementation is bounded
\emph{deterministically} by
\begin{equation}
\begin{aligned}
N(\eps)
&=O\left(
 nk\log\frac1\eps
 +\frac{nk^2\log n\log(1/\eps)}{\eps^2}
 +\frac{nk\log^2(1/\eps)}{\eps^2}
\right)\\
&=O\bigl(nk^2\log n\;\eps^{-2}\log^2(1/\eps)\bigr).
\end{aligned}
\label{eq:oracle-complexity}
\end{equation}
Independence-oracle queries and maximum-weight-base computation are not counted.
\end{lemma}

\begin{proof}
Each RRG run has $\lceil k/2\rceil$ iterations, and each iteration requires
$O(n)$ marginal-value queries to form a maximum-weight base on a ground set of
size at most $2n$. The amplified certificate uses $R=O(\log(1/\rho))=O(\log(1/\eps))$
independent runs, hence $O(nk\log(1/\eps))$ queries. Advanced Preprocessing
performs at most $k$ iterations, each requiring $O(n)$ marginal queries, hence
$O(nk)$ additional queries; since $\eps\le1/2$ this is absorbed by the first term.

For a swap, \eqref{eq:swap-samples} with $\delta_s=\Theta(\eps)$ gives
$N_s=O\bigl((k\log n+\log(1/\eps))/\eps^2\bigr)$.
There are $O(n)$ candidate entering elements, and each sampled marginal
$g(R_\ell\cup\{i\})-g(R_\ell)$ uses two controlled-oracle evaluations; the
$N_s$ sets $R_\ell$ are shared across candidates, so one Poisson event costs
$O(nN_s)$ evaluations. By Lemma~\ref{lem:budget} the number of events is at most
$N_{\max}=O(1+k\log(1/\eps_0)+\log(1/\rho))=O(k\log(1/\eps))$ deterministically.
Multiplying gives
$O\bigl(nk^2\log n\log(1/\eps)\eps^{-2}+nk\log^2(1/\eps)\eps^{-2}\bigr)$.
Adding the RRG and preprocessing costs proves the claim; the second display
follows since $k\ge1$ and $\log(1/\eps)\ge\log 2$.
\end{proof}

\subsection{Combining the Pieces: The Resilience Theorem}
\label{sec:main-theorem}
We now combine the preprocessing and swap results with the unchanged
SGS-Poisson dynamics. We first dispose of the two degenerate ranks.

\begin{lemma}[Small rank]
\label{lem:small-rank}
If $k=0$, returning $\varnothing$ is optimal. If $k=1$, querying $\widehat f$ on
$\varnothing$ and on every feasible singleton and returning the best queried set
uses $O(n)$ queries and yields $f(A_{\rm out})\ge\OPT-2\xierr$.
\end{lemma}

\begin{proof}
For $k=1$ every feasible set is $\varnothing$ or a singleton, so the optimum is
among the queried sets. If $A_{\rm out}$ maximizes $\widehat f$ over that family
and $O$ is optimal, then
$f(A_{\rm out})\ge\widehat f(A_{\rm out})-\xierr\ge\widehat f(O)-\xierr\ge f(O)-2\xierr$.
\end{proof}

\begin{theorem}[Resilience of SGS-Poisson]
\label{thm:main}
Let $f:2^U\to[0,1]$ be nonnegative submodular and let $M=(U,\I)$ be a matroid
of rank $k$. For every $\eps\in(0,1/2]$, supplied error bound
$\xierr\ge0$, and persistent $\xierr$-controlled oracle, the truncated SGS-Poisson value-oracle
implementation of Section~\ref{sec:algorithms}, run with
$\rho=\eps_0=\eps/100$ and $\delta_s=\eps/250$, outputs
$A_{\rm out}\in\I$ satisfying
\begin{equation}
\E[f(A_{\rm out})]
\ge
\begin{cases}
(1/e-\eps)\OPT-Ck\xierr,
& f\text{ non-monotone},\\[1mm]
(1-1/e-\eps)\OPT-Ck\xierr,
& f\text{ monotone},
\end{cases}
\label{eq:resilience-main}
\end{equation}
with $C=8$, and makes at most
\begin{equation}
N(\eps)=
O\left(
nk\log\frac1\eps+
\frac{nk^2\log n\log(1/\eps)}{\eps^2}
+
\frac{nk\log^2(1/\eps)}{\eps^2}
\right)
\label{eq:main-complexity-2}
\end{equation}
oracle calls, deterministically.
\end{theorem}

\begin{proof}
For $k=0$, Lemma~\ref{lem:small-rank} gives the exact optimum without
queries. For $k=1$, its guarantee $\OPT-2\xierr$ is stronger than either
claimed bound. Assume henceforth that $k\ge2$. The error radius $\xierr$
is supplied to the certificate construction in \eqref{eq:certificate-estimator}.

Let
$\mathcal E:=\{\OPT\le\widehat V\le8\OPT+32k\xierr\}$.
By Lemma~\ref{lem:certificate}, $\Prb(\mathcal E)\ge1-\rho$. We take the
initial sigma-field to include the RRG repetitions and their outputs, so
$\mathcal E$ is measurable at time zero. Conditional on this initial
sigma-field, all subsequent preprocessing and SGS-Poisson randomness is
independent of the RRG repetitions. The preprocessing certificate is used in
conditional expectation, while the residual-mass bound
\eqref{eq:residual-contracted} and the swap bound \eqref{eq:eta-bound} hold
\emph{deterministically}, hence for every realized preprocessing trajectory;
the subsequent conditioning on $\mathcal E$ is therefore legitimate.

On $\mathcal E$, Theorem~\ref{thm:preprocess} and \eqref{eq:contracted-certificate} give
\begin{equation}
\mathbb E\left[
g(O_g)+\tfrac12g(\varnothing)\;\middle|\;\mathcal E
\right]
\ge
\OPT-\tfrac18k\xierr,
\label{eq:cert-in-proof}
\end{equation}
and \eqref{eq:eta-bound} gives $\eta\le\frac{13}{25}\eps\OPT+6k\xierr$ uniformly.

\emph{Degenerate contraction.} The contracted rank $k'$ is random, so we
must treat $k'=0$ without conditioning the averaged certificate on that event.
For any realization with $k'=0$, the output is $\bar S\setminus D$ and
$g(O_g)=g(\varnothing)=f(\bar S)$. For either
$c\in\{1/e,1-1/e\}$,
\[
f(\bar S)\ge(1-\eps_0)c\bigl(g(O_g)+\tfrac12g(\varnothing)\bigr)-\eta,
\]
since $\tfrac32c<1$ and $\eta\ge0$. Thus the pointwise lower bound derived
below for $k'\ge1$ also holds when $k'=0$, and can be averaged over all
preprocessing outcomes together.

\emph{Non-monotone case.} Run SGS-Poisson on the contracted instance with
starting time $\eps_0$. Let $S^{\rm full}$ be the output of the untruncated
process. Conditioned on $\mathcal E$ and on the realized $\bar S$,
Proposition~\ref{prop:sgs} gives
\[
\mathbb E[g(S^{\rm full})\mid\bar S,\mathcal E]
\ge
(1-\eps_0)\tfrac1e\, g(O_g)
+\tfrac1e\, g(\varnothing)-\eta
\ge
\frac{1-\eps_0}{e}\Bigl(g(O_g)+\tfrac12g(\varnothing)\Bigr)-\eta,
\]
where we used $g(\varnothing)\ge0$ and $\frac1e\ge\frac{1-\eps_0}{2e}$.
Taking expectation over $\bar S$ given $\mathcal E$ and using
\eqref{eq:cert-in-proof},
\[
\mathbb E[g(S^{\rm full})\mid\mathcal E]
\ge
\frac{1-\eps_0}{e}\Bigl(\OPT-\tfrac18k\xierr\Bigr)-\eta
\ge
\Bigl(\frac1e-\frac{\eps_0}{e}-\frac{13}{25}\eps\Bigr)\OPT
-\Bigl(6+\tfrac1{8e}\Bigr)k\xierr .
\]
With $\eps_0=\eps/100$ this is at least
$(\frac1e-0.53\eps)\OPT-6.1\,k\xierr$.

\emph{Monotone case.} The contraction $g$ is also monotone, and
Proposition~\ref{prop:sgs} gives
$\mathbb E[g(S^{\rm full})\mid\bar S,\mathcal E]\ge(1-\eps_0)(1-\frac1e)g(O_g)+\frac1e g(\varnothing)-\eta$.
Because $\frac1e\ge\frac12(1-\frac1e)\ge\frac{1-\eps_0}{2}(1-\frac1e)$, the same
computation yields
$\mathbb E[g(S^{\rm full})\mid\mathcal E]\ge(1-\frac1e-0.53\eps)\OPT-6.1\,k\xierr$,
using $6+(1-1/e)/8<6.1$.

\emph{Truncation and de-conditioning.} By Lemma~\ref{lem:budget}, replacing
$S^{\rm full}$ by the truncated output costs at most $\rho\OPT=0.01\eps\OPT$.
Writing $c\in\{1/e,1-1/e\}$ and using \eqref{eq:output-identity},
$\mathbb E[f(A_{\rm out})\mid\mathcal E]\ge(c-0.54\eps)\OPT-6.1k\xierr$.
Finally, since $f(A_{\rm out})\ge0$,
\[
\mathbb E[f(A_{\rm out})]
\ge
\Prb(\mathcal E)\,\mathbb E[f(A_{\rm out})\mid\mathcal E]
\ge
(1-\rho)\max\{0,(c-0.54\eps)\OPT-6.1k\xierr\}
\ge
(c-0.55\eps)\OPT-6.1k\xierr,
\]
using $c\le1$ and $\rho=\eps/100$. Both bounds in \eqref{eq:resilience-main}
follow with $C=8$, and in fact with a slack factor of about $1.8$ in the
$\eps\OPT$ term. The complexity bound is Lemma~\ref{lem:oracle-complexity}.
\end{proof}

\begin{remark}[Adaptive queries and empirical oracles]
\label{rem:queried-sets}
The offline theorem is an expectation guarantee for each fixed controlled
oracle. One cannot simply condition it on an empirical-accuracy event that
also depends on the algorithm's random trajectory. The following coupling
justifies its online use under the stochastic model of Section~\ref{sec:setup}.

Preassign independent reward samples to every admissible action, independently
of the offline random seed, and let $\widehat f(S)$ be the empirical mean of
the prescribed number of samples. This table need only be revealed on demand;
it is not constructed by the learner. Define an analytical oracle by clipping
each table entry to $[f(S)-\xierr,f(S)+\xierr]\cap[0,1]$, and extend it by
$f$ outside the admissible action set. Conditional on the table, this is a
fixed persistent controlled oracle, so Theorem~\ref{thm:main} applies without
conditioning on the query-accuracy event.

Couple the real and clipped executions using the same offline seed. They
agree until an inaccurate entry is first queried. If each previously unseen
action has conditional error probability at most $p_{\rm est}$, their
probability of disagreement is at most $N(\eps)p_{\rm est}$ by the
deterministic query bound and a union bound. Both outputs are in $\I$, so
the real output's expected value is at least the clipped output's expected
value minus $N(\eps)p_{\rm est}\OPT$. This proves the needed online
interface using accuracy only along queried actions. The clipping is solely
an analysis device and requires no knowledge of $f$ by the learner.
\end{remark}

\begin{corollary}[Resilience parameters]
\label{cor:resilience-parameters}
In the parameterization of Definition~\ref{def:resilience}, the implementation
of Theorem~\ref{thm:main} is
$(1/e,\;2,\;2,\;\psi,\;8k)$-resilient for non-monotone submodular maximization
and $(1-1/e,\;2,\;2,\;\psi,\;8k)$-resilient for monotone submodular
maximization, with
\[
\psi=O\!\left(nk^2\log n\right)
\qquad(\text{assuming }n\ge2).
\]
\end{corollary}

\begin{proof}
Immediate from Theorem~\ref{thm:main}: by \eqref{eq:oracle-complexity},
$N(\eps)=O(\psi\eps^{-2}\log^2(1/\eps))$ with $\psi=O(nk^2\log n)$, since each
of the three terms in \eqref{eq:main-complexity-2} is at most a constant times
$nk^2\log n\,\eps^{-2}\log^2(1/\eps)$ for $\eps\le1/2$.
\end{proof}

\begin{remark}[The scale of the sensitivity bound]
\label{rem:delta-tight}
A comparison of two bases can involve $\Theta(k)$ distinct set values, so
bounded perturbations can change a marginal-sum comparison by
$\Theta(k\xierr)$. This explains the scale of our upper bound, but does not
prove an $\Omega(k\xierr)$ lower bound on the final output's loss, or on all
offline algorithms. Linear dependence also occurs in several earlier bounds: $\delta=2k$ for \textsc{Greedy} and $\delta=4k$ for
\textsc{RandomSampling} under a cardinality constraint
\citep{Nie2022,FouratiFramework2024}, and $\delta=M+1$ for the greedy algorithm
under a general matroid of rank $M$ \citep{Nie2025KSubmodular}. Our constant
$\delta=8k$ is therefore in the same regime.
\end{remark}

\section{Offline-to-Online CMAB}
\label{sec:cmab}
We now use the resilience theorem only as an offline primitive. The
controlled oracle is an analysis interface for this offline subroutine; it is
not assumed to be directly available to the online learner. In the CMAB
application, the offline value estimates are generated by the exploration
mechanism in the offline-to-online reduction. The online conversion is that of
\citet{FouratiFramework2024}; we adapt its general
$(\alpha,\beta,\gamma,\psi,\delta)$ form to the supplied error radius,
randomized oracle interface, and logarithmic query bound proved here. The query-complexity
exponent $\beta$ determines the final horizon exponent for the learner.

\subsection{The Action Set During Exploration}
\label{sec:action-set}

The reduction answers each offline query for a set $S$ by playing $S$ for
$r^\star$ rounds and returning the empirical mean. Consequently every set the
offline algorithm queries must be a playable action. By
Lemma~\ref{lem:query-feasibility}, the sets queried by our implementation lie
in $\Ipl$ of \eqref{eq:Iplus}, and the swap estimator genuinely leaves $\I$.
We therefore make the following model assumption explicit.

\begin{assumption}[Exploration action set]
\label{ass:action-set}
During exploration the learner may play any $S\in\Ipl$, receiving a reward in
$[0,1]$ with mean $f(S)$. The exploited set is required to be feasible,
$\Theta\in\I$, and the benchmark remains $\OPT=\max_{S\in\I}f(S)$.
\end{assumption}

\begin{lemma}[Regret accounting under enlarged exploration]
\label{lem:action-set-free}
Under Assumption~\ref{ass:action-set}, enlarging the exploration action set
does not increase the per-round exploration-regret bound used in
\citet[Theorem 5.3]{FouratiFramework2024}. This is a model relaxation, not
a feasibility guarantee for the original action set.
\end{lemma}

\begin{proof}
The analysis uses the action set only twice. In the exploration term it bounds
each round's contribution by $\alpha f(S^\star)-\E[f(A_j)]\le\alpha\le1$, which
uses only that rewards lie in $[0,1]$ and holds for $A_j\in\Ipl$. In the
exploitation term it uses the offline guarantee for the returned set $\Theta$,
which is feasible by construction; the benchmark $S^\star$ is unchanged. The
concentration analysis (Hoeffding over the $r^\star$ plays of each queried
action) is likewise indifferent to feasibility of the queried set.
\end{proof}

\begin{remark}[When the relaxation is not available]
If the environment forbids infeasible super-arms, the CMAB corollaries in
this paper cannot be invoked. A rank-$(k-1)$ truncation does not repair this
for a general matroid: even a small independent set can become dependent
when one element is added. Nor are this estimator's queries automatically
feasible for a cardinality constraint. A feasible-query estimator or a
different online algorithm is needed.
\end{remark}

\begin{remark}[Persistence via caching]
\label{rem:caching}
The offline analysis uses a persistent oracle: the same set always receives
the same value. As in the empirical-oracle implementation of
\citet{Nie2023,FouratiFramework2024}, this is obtained by caching, i.e.\ by storing
the empirical mean of each queried action and reusing it if the offline
algorithm queries that action again. Caching can only decrease the number of
exploration rounds, so it does not affect the regret bound. Persistence is used
in Lemma~\ref{lem:right-continuity} and in the fixed-oracle coupling of
Remark~\ref{rem:queried-sets}. Cache keys are the real actions obtained
after removing dummy elements.
\end{remark}

\subsection{Regret Bounds}

\begin{theorem}[Offline-to-online reduction, adapted from \citealp{FouratiFramework2024}]
\label{thm:fourati}
Suppose $\psi\ge1$ and $\delta\ge1$, let $\beta,\gamma\ge0$ be fixed
exponents, and suppose an offline algorithm $\mathcal A(\eps;\xierr)$ is
$(\alpha,\beta,\gamma,\psi,\delta)$-resilient in the sense of
Definition~\ref{def:resilience}. For stochastic CMAB with full-bandit
feedback, all queried actions admissible, and $T\ge\max\{2,\psi\}$,
the explore-then-commit reduction achieves
\[
R_\alpha(T)=\widetilde O\!\left(
\delta^{\frac{2}{3+\beta}}
\psi^{\frac{1}{3+\beta}}
T^{\frac{2+\beta}{3+\beta}}
\right).
\]
The logarithmic dependence on $T$ and the fixed exponent $\gamma$ are
absorbed into $\widetilde O(\cdot)$. The benchmark is $\alpha\OPT$, not
$(\alpha-\eps)\OPT$.
\end{theorem}

\begin{proof}[Verification of the interface adaptation]
The explore-then-commit reduction and its regret exponents are due to
\citet[Theorem 5.3]{FouratiFramework2024}. We verify the adjustments needed
here: an advertised error radius, a randomized adaptive subroutine, and
$\gamma=2$ rather than the original $\gamma\in\{0,1\}$.
Put $L_T=\log(2T)$ and, for a sufficiently large constant $K$, choose
\[
\eps=K\max\left\{
\left(\frac{\psi L_T^\gamma}{T}\right)^{1/(\beta+1)},
\left(\frac{\psi\delta^2 L_T^{\gamma+1}}{T}\right)^{1/(\beta+3)}
\right\}.
\]
If $\eps>1/2$, use $R_\alpha(T)\le T$. Otherwise supply the routine with
$\xierr=\eps/\delta$ and answer each new queried action with the cached mean
of $r_{\rm est}=\lceil2\delta^2L_T/\eps^2\rceil$ independent rewards, as in
\citet{Nie2023,FouratiFramework2024}. Enforce the known query budget and
check query admissibility and output feasibility, returning a fixed feasible
fallback on a violation. Controlled-oracle executions are unaffected.
Hoeffding's inequality \citep{Hoeffding1963} bounds each new action's error
probability by $2e^{-4L_T}\le T^{-2}$. With the deterministic query bound
$Q=O(\psi\eps^{-\beta}L_T^\gamma)$, our tuning ensures
\[
Qr_{\rm est}
=O\!\left(\psi\eps^{-\beta}L_T^\gamma+
\psi\delta^2\eps^{-(\beta+2)}L_T^{\gamma+1}\right)
\le\eps T.
\]
The clipped-oracle coupling in Remark~\ref{rem:queried-sets}, rather than
conditioning the offline theorem on a trajectory-dependent accuracy event,
then gives
\[
\alpha\OPT-\E[f(\Theta)]\le\eps\OPT+\delta\xierr+Q/T^2.
\]
For the possibly random exploration duration $\tau_{\rm expl}\le Qr_{\rm est}$,
nonnegative exploration rewards and $f(\Theta)\le1$ imply
\[
\E\!\left[\sum_{t=1}^T Y_t\right]
\ge\E[(T-\tau_{\rm expl})f(\Theta)]
\ge T\E[f(\Theta)]-Qr_{\rm est}.
\]
Consequently $R_\alpha(T)\le3\eps T+1$; no independence between $\Theta$
and $\tau_{\rm expl}$ is needed. The same argument applies to each
deterministic prefix $s\le T$ of this $T$-tuned run. Define $\Theta$ by
completing the exploration on the preassigned reward table even if the
observed prefix ends earlier. The exploitation contribution is at least
$\E[(s-\tau_{\rm expl})_+f(\Theta)]$, where $(x)_+=\max\{x,0\}$, so
\begin{equation}
R_\alpha(s;T)
\le Qr_{\rm est}+s\left(\eps\OPT+\delta\xierr+Q/T^2\right)
\le3\eps T+1.
\label{eq:prefix-regret}
\end{equation}
Here $R_\alpha(s;T)$ denotes regret through round $s$ for parameters tuned
to $T$, not to $s$. If the tuned $\eps>1/2$, the trivial bound
$R_\alpha(s;T)\le s\le T$ gives the same conclusion. This establishes
the prefix control used in Remark~\ref{rem:anytime} without a sign assumption
on instantaneous $\alpha$-regret. Finally, since $T\ge\psi$ and $\delta\ge1$,
\[
\left(\frac{\psi}{T}\right)^{1/(\beta+1)}
\le\left(\frac{\psi\delta^2}{T}\right)^{1/(\beta+3)},
\]
so the second tuning term dominates up to logarithmic factors, giving the
stated rate, also in the trivial $\eps>1/2$ regime.
\end{proof}

\begin{corollary}[Full-bandit CMAB under general matroids]
\label{cor:cmab}
Consider stochastic full-bandit CMAB with $n\ge2$, a general rank-$k$
matroid with $k\ge1$, a
nonnegative submodular mean reward, and Assumption~\ref{ass:action-set}. Let
$\psi=O(nk^2\log n)$ and $\delta=8k$. Then for every horizon $T\ge\max\{2,\psi\}$ there
exist algorithms with
\[
R_{1/e}(T)
=
\widetilde O\!\left(n^{1/5}k^{4/5}T^{4/5}\right)
\quad\text{(non-monotone)},
\qquad
R_{1-1/e}(T)
=
\widetilde O\!\left(n^{1/5}k^{4/5}T^{4/5}\right)
\quad\text{(monotone)} .
\]
\end{corollary}

\begin{proof}
Theorem~\ref{thm:main} and Corollary~\ref{cor:resilience-parameters} give
$\beta=2$, $\gamma=2$ (admissible by Remark~\ref{rem:gamma}),
$\psi=O(nk^2\log n)$ and $\delta=8k\ge1$, with
$\alpha=1/e$ in the non-monotone case and $\alpha=1-1/e$ in the monotone case;
Lemma~\ref{lem:query-feasibility}, Lemma~\ref{lem:action-set-free},
Remark~\ref{rem:caching} and Lemma~\ref{lem:budget} discharge the interface
conditions. Since $\delta=8k\ge1$, Theorem~\ref{thm:fourati} applies for
$T\ge\max\{2,\psi\}$. Applying
Theorem~\ref{thm:fourati} with $\beta=2$,
\[
R_\alpha(T)
=\widetilde O\!\left(
k^{2/5}\,(nk^2\log n)^{1/5}\,T^{4/5}
\right)
=\widetilde O\!\left(n^{1/5}k^{4/5}T^{4/5}\right),
\]
where $\widetilde O$ absorbs the $\log^{1/5}n$ factor and the polylogarithmic
factors in $T$. For $T<\psi$ the bound holds trivially after enlarging its
universal constant, since $R_\alpha(T)\le T$ under the present normalization.
\end{proof}

\section{Partition Matroids: a Better Rate}
\label{sec:partition}

Partition matroids cover several of the motivating applications (submodular
welfare, ad allocation, group-fair sensor placement, influence maximization with
per-community budgets). For their unit-capacity versions we improve the rank
dependence of the regret
from $k^{4/5}$ to $k^{3/5}$. The improvement is not an artifact of a different
analysis: it comes from the fact that, for a partition matroid, a
maximum-weight base under the estimated weights is a \emph{transversal} of
per-part maximizers, so a swap can be simulated exactly while evaluating
marginals in only one part.

Throughout this section $M=(U,\I)$ is a partition matroid of rank $k$: the ground
set $U$ is partitioned into $P_1,\ldots,P_k$ and $S\in\I$ iff $|S\cap P_j|\le1$
for every $j$. Generalized partition matroids are discussed in
Remark~\ref{rem:generalized-partition}.

\subsection{Per-Part Augmentation}

The rank-$k$ truncation used in Section~\ref{sec:dummy} destroys the partition
structure, so we augment per part instead: let $d_j$ be a fresh dummy element and
put $P_j^+:=P_j\cup\{d_j\}$, $D:=\{d_1,\ldots,d_k\}$, $U^+:=U\cup D$, again with
unit capacities, and extend $f,\widehat f$ by $f^+(S):=f(S\setminus D)$ and
$\widehat f^+(S):=\widehat f(S\setminus D)$ as before.

\begin{lemma}[Per-part augmentation]
\label{lem:partition-dummy}
$M^+=(U^+,\mathcal I^+)$ is a unit-capacity partition matroid of rank $k$ on a
ground set of size $n+k\le2n$; $\widehat f^+$ is a $\xierr$-controlled oracle for
$f^+$; every contraction $M^+/Q$ with $Q\in\mathcal I^+$, after deleting loops, is
again a unit-capacity partition matroid; and the three conclusions of
Lemma~\ref{lem:dummy} hold for $M^+$.
\end{lemma}

\begin{proof}
The parts of $M^+$ still partition $U^+$ and there are $k$ of them, so the rank is
$k$ and $|U^+|=n+k$. Contracting at $Q$ deletes every part met by $Q$ (its
elements become loops) and leaves the remaining parts with capacity $1$, so
$M^+/Q$ is again a unit-capacity partition matroid, on the parts not met by $Q$.
For Lemma~\ref{lem:dummy}(a), a residual independent set $T$ is padded to a
residual base by adding $d_j$ for each residual part $j$ with $T\cap P_j^+=\varnothing$;
this is possible because a residual part is by definition not met by $Q$, so its
dummy is unused. Parts (b) and (c) follow by the same padding, every dummy having
zero $f^+$- and $\widehat f^+$-marginal.
\end{proof}

Consequently, the analysis of Sections~\ref{sec:robust-rrg}--\ref{sec:robust-swap}
carries over unchanged to a partition matroid with this augmentation in place
of the truncated one: RRG,
the certificate $\widehat V$, Advanced Preprocessing, the drift lemma, and
Lemma~\ref{lem:swap} use the matroid exchange properties together with the
dummy-completion conclusions of Lemma~\ref{lem:dummy}. In
particular Theorem~\ref{thm:main} holds as stated.

\subsection{The Lazy Swap and an Exact Simulation}

Write $M'=M^+/\bar S$ for the contracted matroid, a unit-capacity partition
matroid with parts $P'_1,\ldots,P'_{k'}$, where $k'$ is the contracted rank, and
put $n':=\sum_{j=1}^{k'}|P'_j|\le2n$. If $k'=0$, return $\bar S\setminus D$
without running either swap procedure.

\begin{algorithm}[ht]
\caption{Lazy partition swap with a controlled oracle (adapted from \citealp[Algorithm 5]{Kulik2026})}
\label{alg:lazy-swap}
\begin{algorithmic}[1]
\Require time $t$, current set $A\in\I(M')$, sample size $N_s$
\State Sample $R_1,\ldots,R_{N_s}\sim t\one_A$ independently
\State Sample a part index $j$ uniformly from $\{1,\ldots,k'\}$
\State For every $i\in P'_j$ compute
       $\widehat w_i=\frac1{N_s}\sum_{\ell=1}^{N_s}\bigl[\widehat g(R_\ell\cup\{i\})-\widehat g(R_\ell)\bigr]$
\State $J\gets\argmax_{i\in P'_j}\widehat w_i$, ties broken by the fixed order
\State $I\gets$ the unique element of $A\cap P'_j$ if $A\cap P'_j\neq\varnothing$,
       and $I\gets\bot$ otherwise
\State \Return $(I,J)$
\end{algorithmic}
\end{algorithm}

\begin{lemma}[Exact simulation]
\label{lem:simulation}
The simulation of \citet[Section 4.2.2]{Kulik2026} gives the following identity.
For a unit-capacity partition matroid, Algorithm~\ref{alg:lazy-swap} and the
general controlled-oracle swap of Lemma~\ref{lem:swap}, using the canonical
within-part exchange map, induce the same joint
distribution of $(I,J)$, given the same time $t$, the same current set $A$, and
the same fixed element order used for lexicographic base tie-breaking.
\end{lemma}

This is the deterministic simulation argument of
\citet[Section 4.2.2, preceding Lemma 4.8]{Kulik2026}, so we do not repeat its
proof. The identity holds for every fixed collection of estimated weights;
replacing exact values by $\widehat g$ therefore does not require
submodularity of the surrogate. The new guarantees here are the robustness
of those weights and the deterministic query-cost bound below.

Because the output distribution is identical, every conclusion of
Lemma~\ref{lem:swap} --- validity, right-continuity, and
\begin{equation}
\eta\le\tfrac45\delta_sL+4k\xierr
\label{eq:partition-eta}
\end{equation}
with the same $N_s$ of \eqref{eq:swap-samples} --- transfers to
Algorithm~\ref{alg:lazy-swap} without any new concentration argument. This is the
point of the simulation: the uniform event $\mathcal E_{\rm sw}$ of
Lemma~\ref{lem:swap} concerns the population base sums $W_g(Z)$ over all
independent sets and is a statement about the shared samples
$R_1,\ldots,R_{N_s}$, not about which marginals the algorithm chooses to evaluate.

\subsection{A Deterministic Budget for the Lazy Swap}

Algorithm~\ref{alg:lazy-swap} uses at most $N_s(|P'_j|+1)$ controlled-oracle
calls. We charge this upper bound to a swap in part $j$; caching or skipping
dummy evaluations can only reduce the actual number of calls. The random
part sizes require a query-cost cap in addition to the event cap.

\begin{lemma}[Partition query budget]
\label{lem:partition-budget}
Assume $k'\ge1$, $\eps_0\in(0,1]$, and $\rho\in(0,1)$. Apply
Lemma~\ref{lem:budget} with $\Lambda=k'\ln(1/\eps_0)$, the exact mean number
of Poisson events on the contracted instance, and let $N_{\max}$ be the
resulting event cap. Put
\[
B_q:=N_s\left[\,2N_{\max}\Bigl(\frac{n'}{k'}+1\Bigr)+3(n'+1)\ln\frac1\rho\,\right].
\]
Before making a swap's oracle calls, abort and return $\bar S\setminus D$
if its charge would bring the total above $B_q$. Then the swap stage makes
at most $B_q$ calls deterministically, where
\[
B_q=O\!\left(N_s n\left[1+\log\frac1{\eps_0}+\log\frac1\rho\right]\right),
\]
and the additional abort costs at most $\rho\OPT$ in expectation, also
conditionally on the preprocessing history. In particular, $B_q=O(N_s n\log(1/\eps))$ when both $\eps_0$ and $\rho$ are
$\Theta(\eps)$ for $\eps\in(0,1/2]$; their constants of proportionality
need not coincide.
\end{lemma}

\begin{proof}
Condition on the preprocessing outcome, so the residual parts, $k'$, and
$n'$ are fixed. All probabilities and expectations below are conditional
on this outcome. Let $N\le N_{\max}$ be the number of swap calls in the
event-capped process before imposing the query-cost cap, and let
$Y_\ell:=|P'_{j_\ell}|+1$ be the charge of event $\ell$ in units of $N_s$.
The $j_\ell$ are i.i.d.\ uniform on $\{1,\ldots,k'\}$ and independent of
the Poisson clock. Since $N$ depends only on that clock and its event cap,
conditioning further on $N=\nu$ preserves their independence and gives
\[
\E[Y_\ell\mid N=\nu]=1+\frac{n'}{k'}\qquad(\ell\le\nu).
\]
For $\nu=0$ the charge is zero. Otherwise put
$X=\sum_{\ell=1}^{\nu}Y_\ell/(n'+1)$ and
$\mu_\nu=\nu(1+n'/k')/(n'+1)$. The multiplicative Chernoff bound for
independent $[0,1]$-valued summands gives
\[
\Prb\!\left(X\ge2\mu_\nu+3\ln(1/\rho)\mid N=\nu\right)\le\rho.
\]
Indeed, writing $(1+\lambda)\mu_\nu=2\mu_\nu+3\ln(1/\rho)$ gives
$\lambda\ge1$ and a bound
$\exp(-\lambda\mu_\nu/3)\le\rho$. Multiplying the threshold by
$N_s(n'+1)$ and using $\nu\le N_{\max}$ shows that the probability of
total charge exceeding $B_q$ is at most $\rho$, uniformly in $\nu$ and
in the preprocessing outcome. Averaging over $N$ preserves the bound for
every such outcome.
The coupled executions with and without the additional cost cap agree
unless that charge is exceeded. The event-capped output is feasible and
has value at most $\OPT$, while the fallback is nonnegative. Thus the
additional conditional expected loss is at most $\rho\OPT$, and averaging
over preprocessing gives the unconditional bound.

For the size bound, the ceiling in the event cap gives
$N_{\max}=O(1+k'\log(1/\eps_0)+\log(1/\rho))$, hence
\[
\begin{aligned}
N_{\max}\Bigl(\frac{n'}{k'}+1\Bigr)
&=O\!\left((n'+k')\log\frac1{\eps_0}
 +\Bigl(\frac{n'}{k'}+1\Bigr)\left[1+\log\frac1\rho\right]\right)\\
&=O\!\left(n\left[1+\log\frac1{\eps_0}+\log\frac1\rho\right]\right),
\end{aligned}
\]
using $1\le k'\le k\le n$ and $n'\le2n$. The remaining term in $B_q$
has the same order. Setting $\eps_0=\Theta(\eps)$ and
$\rho=\Theta(\eps)$ separately yields the claimed specialization.
The improvement comes from combining $k'$-scale event counts with
$n'/k'$-scale mean charges, rather than paying for all elements at every event.
\end{proof}

\subsection{The Partition-Matroid Theorem}

\begin{theorem}[Partition matroids]
\label{thm:partition}
Let $n\ge2$ and let $M$ be a unit-capacity partition matroid of rank $k\ge1$.
Run the implementation
of Theorem~\ref{thm:main} with the per-part augmentation of
Lemma~\ref{lem:partition-dummy}, the lazy swap of Algorithm~\ref{alg:lazy-swap},
and the two truncations of Lemmas~\ref{lem:budget} and~\ref{lem:partition-budget}
at level $\rho_{\rm tr}=\eps/200$ each, keeping the certificate failure
probability at $\rho=\eps/100$. Then for every $\eps\in(0,1/2]$ and every
persistent $\xierr$-controlled oracle the output $A_{\rm out}\in\I$ satisfies
\eqref{eq:resilience-main} with $C=8$, and the number of controlled-oracle calls
is at most
\[
O\!\left(
nk\log\frac1\eps
+\frac{nk\log n\log(1/\eps)}{\eps^2}
+\frac{n\log^2(1/\eps)}{\eps^2}
\right)
\]
deterministically. Hence the implementation is
$\bigl(1/e,2,2,\psi,8k\bigr)$- and $\bigl(1-1/e,2,2,\psi,8k\bigr)$-resilient with
$\psi=O(nk\log n)$, and under Assumption~\ref{ass:action-set}, for $T\ge\max\{2,\psi\}$,
\[
R_{1/e}(T)=\widetilde O\!\left(n^{1/5}k^{3/5}T^{4/5}\right),
\qquad
R_{1-1/e}(T)=\widetilde O\!\left(n^{1/5}k^{3/5}T^{4/5}\right),
\]
respectively for non-monotone and monotone objectives.
\end{theorem}

\begin{proof}
\emph{Approximation.} Lemma~\ref{lem:partition-dummy} allows us to reuse
Section~\ref{sec:robustness}. Lemma~\ref{lem:simulation} identifies the lazy
swap's output distribution with that of the general swap, so
\eqref{eq:partition-eta} holds and the proof of Theorem~\ref{thm:main} carries
over after accounting for the additional query-budget truncation. The only
change is that two truncations are in force rather
than one; at level $\rho_{\rm tr}=\eps/200$ each, their combined expected loss is
at most $2\rho_{\rm tr}\OPT=0.01\eps\OPT$, which is exactly the budget already
allotted to truncation in that proof. The certificate event $\mathcal E$ still
has failure probability at most $\eps/100$, by taking
$R=\lceil14\ln(100/\eps)\rceil$ RRG repetitions. Hence the constant chain, and
$C=8$, are unaffected.

\emph{Complexity.} RRG and Advanced Preprocessing cost $O(nk\log(1/\eps))$ as in
Lemma~\ref{lem:oracle-complexity}. The swap stage costs at most
$B_q=O(N_s n\log(1/\eps))$ by Lemma~\ref{lem:partition-budget}, and
\eqref{eq:swap-samples} with $\delta_s=\Theta(\eps)$ gives
$N_s=O\bigl((k\log n+\log(1/\eps))/\eps^2\bigr)$, so
$B_q=O\bigl(nk\log n\log(1/\eps)\eps^{-2}+n\log^2(1/\eps)\eps^{-2}\bigr)$.
Summing gives the display, which is $O(\psi\eps^{-2}\log^2(1/\eps))$ with
$\psi=O(nk\log n)$.

\emph{Regret.} Feed $\beta=2$, $\gamma=2$, $\psi=O(nk\log n)$ and $\delta=8k$
into Theorem~\ref{thm:fourati}, exactly as in Corollary~\ref{cor:cmab}; the
interface conditions are discharged by Lemmas~\ref{lem:query-feasibility}
and~\ref{lem:action-set-free} and Remarks~\ref{rem:caching}
and~\ref{rem:queried-sets}, none of which depends on the matroid family. This
gives $k^{2/5}(nk\log n)^{1/5}T^{4/5}=\widetilde O(n^{1/5}k^{3/5}T^{4/5})$.
\end{proof}

The saving is exactly one factor of $k$ in $\psi$: only $O(n/k')$ marginals are
evaluated per swap in expectation instead of $O(n)$, while the number of samples $N_s$ per
evaluated marginal is unchanged, because the concentration event those samples
support is the same one.

\begin{remark}[Generalized partition matroids]
\label{rem:generalized-partition}
The improved bound proved above is for unit capacities. With capacities
$\ell_j>1$, selecting a base element uniformly chooses part $j$ with
probability $\ell'_j/k'$, where $\ell'_j$ is its residual capacity.
Evaluating every candidate in the chosen part then costs, in expectation,
$O(N_s\sum_j\ell'_j|P'_j|/k')$ oracle calls per event. This need not be
$O(N_s n/k')$; for a single large-capacity part it is $O(N_s n)$. Consequently the
unit-capacity improvement does not automatically extend to arbitrary
partition capacities. The general-matroid guarantee of
Theorem~\ref{thm:main} remains available for those instances.
\end{remark}

\section{Conclusion}

We established controlled-oracle resilience of SGS-Poisson for nonnegative
submodular maximization over general matroids, with limiting approximation
factors $1/e$ and $1-1/e$, sensitivity $8k\xierr$, and a deterministic
$\widetilde O(nk^2\eps^{-2})$ query budget. The main ingredients are
adaptive preprocessing along the noisy trajectory, robust base-sum swap
comparisons, and explicit query-budget and query-domain guarantees.
The resulting full-bandit regret is
$\widetilde O(n^{1/5}k^{4/5}T^{4/5})$, improving to
$\widetilde O(n^{1/5}k^{3/5}T^{4/5})$ for unit-capacity partition matroids.
Both online results permit exploration in $\Ipl$; obtaining the same
guarantees with strictly feasible exploration remains a direction for future work.

After our first arXiv version (arXiv:2608.12134v1, submitted August~12, 2026),
\citet{WanZhang2026} submitted their work on August~25, 2026, obtaining
$\widetilde O(n^{1/3}k^{2/3}T^{2/3})$ $(1-1/e)$-regret with strictly feasible
queries in an adversarial model. Their model assumes a normalized, monotone
submodular reward function at \emph{every round}. By contrast, our stochastic
model requires submodularity only of the \emph{mean} reward function:
individual noisy observations and empirical value tables need not be
submodular or monotone. Thus their adversarial model does not subsume the
stochastic model considered here, and our setting is not a special case of
their stated assumptions. In particular, their stated guarantee does not
automatically transfer to our noisy-feedback setting without an additional
analysis. Our work instead establishes resilience to arbitrary bounded oracle
perturbations and derives learning guarantees for both monotone and
non-monotone mean rewards.

\bibliographystyle{abbrvnat}
\bibliography{refs}

\end{document}